\documentclass[twocolumn, 10pt]{article}
\usepackage{geometry}
[        a4paper,
        paperheight=11in,
		textwidth=18cm,
		textheight=23.5cm,
		twocolumn,
		columnsep=7mm,
		centering,
		headsep=.6cm,
]
\usepackage{amsmath, amsthm, amssymb, amsfonts}

\usepackage{stix2}[lcgreekalpha]
\usepackage{tikz} %
\usepackage{booktabs} %
\usepackage{fvextra} %
\usepackage{array} %
\usepackage{tabularx} %
\usepackage{soul}
\newtheorem{innercustom}{Requirement}
\newenvironment{crequirement}[1]
{\renewcommand\theinnercustom{#1}\innercustom}
  {\endinnercustom}

\newcommand{\bR}{R}
\newcommand{\bH}{S}

\usepackage{tcolorbox}

\newtheorem{proposition}{Proposition}[section]
\newtheorem{lemma}{Lemma}[section]
\newtheorem{corollary}{Corollary}[section]
\theoremstyle{definition}
\newtheorem{definition}{Definition}[section]
\newtheorem{example}{Example}[section]
\usepackage{authblk}

\title{The AI off-switch problem as a signalling game: bounded rationality and incomparability}

\author[1]{Alessio Benavoli}
\author[2]{Alessandro Facchini}
\author[2]{Marco Zaffalon}

\affil[1]{School of Computer Science and Statistic, Trinity College Dublin,  Ireland}
\affil[2]{SUPSI, IDSIA - Dalle Molle Institute for Artificial Intelligence, Lugano, Switzerland.}

\begin{document}

\maketitle

\begin{abstract}
The off-switch problem is a critical challenge in AI control: if an AI system resists being switched off, it poses a significant risk. In this paper, we model the off-switch problem as a signalling game, where a  human decision-maker  communicates its preferences about some underlying decision problem to an AI agent, which then selects actions to maximise the human's utility. We assume that the human is a bounded rational agent and explore various bounded rationality mechanisms. Using real machine learning models, we reprove prior  results and demonstrate that a necessary condition for an AI system to refrain from disabling its off-switch is its uncertainty about the human's utility. We also analyse how message costs influence optimal strategies and extend the analysis to scenarios involving incomparability.
\end{abstract}

\section{Introduction}
With the rapid advancements in AI, the challenge of designing systems that are beneficial to humans is becoming increasingly critical. In the rest of the paper, we will refer to an AI system as a \emph{robot} and assume that it is an agent seeking to maximise a utility function $\nu$.
A central concern in AI safety is ensuring that the utility function 
$\nu$ aligns with human interests. If a robot's goals conflict with human values, it could make harmful or even adversarial decisions. Moreover, the robot might resist any attempts by its human creators to modify its utility function or allow them to switch it off

To ensure alignment with human values, from the perspective of utilitarianism\footnote{This is the normative moral theoretical perspective adopted in the works concerned in this article, see e.g. \cite{russell2019human}. The fact that we are restricting to this perspective in our paper does not implies its endorsement (we may actually find such restriction problematic, but a philosophical analysis of the general problem of AI alignment is not the subject of this paper).}, robots should be designed to maximise humans' utilities (preferences) rather than pursue their own goals. This led \cite{russell2019human} to introduce the concept of \textit{AI assistance game}, where an AI agent simply becomes a personal AI assistant. In this setting, AI alignment can be formalised through three key principles \cite{russell2019human}:
\begin{description}
\item[P1:] The Robot's only objective is to maximise the realisation of
human preferences.
\item[P2:] The Robot is initially uncertain about what those preferences are.
\item[P3:] The ultimate source of information about human preferences is human behaviour.
\end{description}
However, designing a Robot to follow these principles does not guarantee that we will be able to control it. The ability to control the Robot ultimately depends on our capacity to switch it off. Will the Robot allow that?
The off-switch problem lies at the heart of the control challenge for AI systems. If a machine cannot be switched off because it actively prevents us from doing so, it poses a serious threat.  In a seminal paper  \cite{hadfield2017off} formulated the \text{off-switch problem} as a game between a Robot (R) and a Human (S for Sapiens).
$\bR$ has three possible actions: (1) take an immediate decision (about some underlying decision problem); (2) defer the decision to $\bH$; (3) do nothing. If $\bR$ selects actions (2), then $\bH$ can either allow $\bR$ to implement the decision or switch $\bH$ off.
The authors in \cite{hadfield2017off} showed that the best action for $\bR$ depends on $\bR$'s uncertainty about  $\bH$'s utility and the rationality of $\bH$.
If $\bR$ is too certain about what $\bH$ wants, and
it knows $\bH$ to be `irrational', then it will have less incentive to defer to $\bH$ (or, equivalently, to allow $\bH$ to switch it off).

The analysis of the off-switch game by \cite{hadfield2017off} was not fully developed within a formal game-theoretic framework. To address this gap, \cite{wangberg2017game} reformulated the off-switch problem as a Bayesian game with incomplete information. In this formulation, Nature determines whether $\bR$ is uncertain with probability $p_u$ and whether $\bH$ is rational with probability $p_r$. Although this is a more rigorous formulation of the off-switch problem from a game-theoretic perspective, it introduces an artificial setting. For example, in \cite{wangberg2017game}, a non-rational  human is modelled as an agent minimising utility, which is a counterintuitive assumption. Moreover, the formulation in \cite{wangberg2017game} does not really model the AI-assistance game principles P1--P3, since $\bR$  does not learn  $\bH$'s preferences.

In this paper, we revisit the off-switch problem and model it more correctly as a \textit{signalling game}. Signalling games \cite{spence1974market,gibbons1992game} specifically refer to a class of two-player Bayesian games with incomplete information, where one player ($\bH$ in our case) possesses private information, while the other ($\bR$ in our case) does not. The informed player shares information with the uniformed one through a message.
We make the assumption that, in the off-switch problem, the messages are $\bH$'s preferences about some underlying decision problem. $\bR$ uses these preferences to learn $\bH$'s utilities and then chooses its optimal action.

In this setting, $\bR$'s uncertainty arises statically from the task of learning from preferences. For $\bH$, we adopt the more realistic assumption that $\bH$ is a \textit{bounded rational} agent, that is an agent who behaves rationally within the limits of their cognitive abilities. We consider different \textit{bounded rationality} mechanisms.

We reprove the results of \cite{hadfield2017off} in this setting using real machine learning models to learn from preferences. We additionally show that  $\bH$ does not have any incentive to lie when sending a message to $\bR$. Furthermore,  we also discuss how the cost of sending a message affects the optimal strategy in the game. 
Both \cite{hadfield2017off} and \cite{wangberg2017game}  consider a single-utility scenario. In contrast, we also analyse scenarios involving incomparability. Such incomparability may arise due to bounded rationality or because $\bH$ has multiple utility functions in mind.

\section{Preliminary}
We assume that the human $\bH$ and the robot $\bR$ aim to maximise $\bH$'s preferences of an underlying decision problem.
We consider a standard set up for decision-making \cite{savage1972foundations} consisting of three key components: 1) a set $\mathcal{S}$ of finite states of the world; 2) a set $\mathcal{O}$ of outcomes; and 3) a set $\mathcal{X}$ of acts, which are mappings from states to outcomes. $\bH$ expresses preferences over acts. We use the term preferences here in a more colloquial sense; it is, in fact, more realistic to assume that we can only observe choices over acts. Therefore, we can represent this decision making problem as a tuple $(\mathcal{S},\mathcal{O},\mathcal{X},C)$, where $C:\mathcal{P}(\mathcal{X})\rightarrow \mathcal{P}(\mathcal{X})$ is a choice function  ($\mathcal{P}(\mathcal{X})$ being the power-set of $\mathcal{X}$). 
\begin{example}
Assume $\bH$'s objective is to make a very good risotto. The outcome  is the taste of the risotto, which is determined by its recipe. The decision-making problem is to find the best recipe. In this context, the state of the world can represent factors such as the quality of the ingredients, e.g., whether they are good or bad, which may significantly influence the resulting taste. Therefore, an act is represented by a vector specifying the recipe for each possible state of the world. If we assume that a recipe is a vector of $\mathbb{R}^c$, where $c$ is the number of ingredients, then an act ${z}$ is a vector in $\mathbb{R}^{2c}$, assuming only two states of the world good/bad. Then $\bH$ expresses choices over the powerset of a finite subset   of $\mathbb{R}^{2c}$ (that is, over  subsets of recipes).
\end{example}
Hereafter, we also include examples of more standard instantiations of this problem, commonly considered in the foundations of decision-making under uncertainty. \textit{Desirable gambles:} $\mathcal{S}=\Omega$ is a possibility space of an uncertain experiment (for instance, tossing a coin); $\mathcal{O}=\mathbb{R}$ and $\mathcal{X}=\mathbb{R}^{|\Omega|}$ is the set of all gambles \cite{walley1991statistical}. \textit{Anscombe-Aumann:} $\mathcal{S}=\Omega$; $\mathcal{O}=D(\mathcal{O})$ is the set of simple probability distributions on a finite set of prizes $\mathcal{O}$ and $\mathcal{X} \subset \mathbb{R}^{|\Omega|\times |\mathcal{O}|}$ \cite{anscombe1963definition}.
A connection between these two settings is proven in \cite{zaffalon2017axiomatising,zaffalon2018desirability}.

For a given decision problem $(\mathcal{S},\mathcal{O},\mathcal{X},C)$, \textit{rationality} of the decision-maker can be imposed by enforcing the choice function $C$ to satisfy a set of constraints \cite{arrow1959rational,sen1994formulation,aizerman1981general}. A minimal rationality requirement is usually:
\begin{equation}
\label{eq:Path}
    C(A \cup B)= C(C(A)\cup B),
\end{equation}
for all $A,B \subseteq \mathcal{P}(\mathcal{X})$, which is known as \textit{path-independence} property \cite{plott1973path}. Other constraints can be added depending on the specific decision-making problem \cite{seidenfeld2010coherent,de2019interpreting,van2018coherent} (including topological constraints).
The choice mechanisms (procedures) are typically defined by specifying a rule that determines how to identify $C(A)$ for each input set $A$.
A first example is  the `scalar optimisation choice' \cite{aizerman1981general} defined as:
\begin{equation}
\label{eq:scalaru}
 C(A)=\{z \in A: \text{there is no $y \in A$ s.t.\ }  \nu(y) > \nu(z)\},   
\end{equation}
that is, the choice function is defined by  $\nu:\mathcal{X} \rightarrow \mathbb{R}$.
We refer to $\nu$ as utility function. Another example is the  `vector optimisation choice' defined through the following Pareto-dominance criterion
\begin{equation}
\label{eq:pareto}
 C(A)=\{z \in A: \text{there is no $y \in A$ s.t.\ }  \boldsymbol{\nu}(y) > \boldsymbol{\nu}(z)\},
\end{equation}
where $\boldsymbol{\nu}:\mathcal{X} \rightarrow \mathbb{R}^d$ is a vectorial function.  The logic of the choice mechanisms \eqref{eq:scalaru} and \eqref{eq:pareto} is based on pairwise-comparisons. A choice function which is not defined by pairwise comparisons is:
 \begin{equation}
\label{eq:pseudoratio}
    C(A)=\left\{\bigcup\limits_{k=1,\dots,d} \arg\max_{z \in A} \nu_k(z)\right\}.
 \end{equation}
 The definitions \eqref{eq:pareto} and \eqref{eq:pseudoratio} coincide with \textit{maximality} and \textit{e-admissibility} when applied to gambles \cite{seidenfeld2010coherent,de2020archimedean}.
Each one of these representations defines choice functions satisfying different `rationality constraints'. For instance, all these three choice functions satisfy  property \eqref{eq:Path}.\footnote{This holds when $\mathcal{X}$ is finite. In the infinite case, additional assumptions are needed.}
We refer to the 
choice function in \eqref{eq:scalaru} as a binary preference.

Note that, here, we do not assume state-independence; therefore, the utility function $\boldsymbol{\nu}$ is generally state-dependent \cite{schervish1990state}. 

\subsection{Learning from choices under rationality}
\label{sec:learning}

The problem of learning from choices generalises the concept of learning from preferences. Given observations in the form of choices
\begin{equation}
\label{eq:choiciedata}
\mathcal{D}=(A_i,C(A_i))_{i=1}^n,
\end{equation}
and assuming the choice function satisfies rationality criteria such as \eqref{eq:Path}, then it can be represented through a mechanism such as \eqref{eq:scalaru}--\eqref{eq:pseudoratio}. In this case, learning a choice function reduces to the problem of learning a utility function vector \cite{benavoli2023learning,benavoli2024tutorial}. Since the unknown is a function we can use a Gaussian Process (GP) to learn $\boldsymbol{\nu}$ \cite{benavoli2023learning,benavoli2024tutorial}. In this section, we assume that for each $y,z \in \mathcal{X}$ then $\boldsymbol{\nu}(y)\neq \boldsymbol{\nu}(z)$. This assumption prevents issues with zero probability when calculating the posterior, as discussed in \cite{benavoli2024tutorial}. As a result, for instance, the choice function defined in \eqref{eq:scalaru} 
is such that $C(A)$ contains only a single element for each set $A$. In the next section, we will generalise this setting by introducing a limit of discernibility.

To learn a choice function, we  assume a GP prior on the latent unknown utility vector function:
\begin{equation}
\label{eq:multiprior3prior}
\boldsymbol{\nu}(z)=\begin{bmatrix}
\nu_1(z) \\ 
\nu_2(z)\\ 
\vdots \\
\nu_d(z)\\ 
\end{bmatrix} \sim GP\left(\mu_0(z),K_0(z,z')\right), 
\end{equation}
where $\mu_0(z),K_0(z,z')$ are respectively, the prior mean and covariance functions (the parameters of the GP), and then use the data in \eqref{eq:choiciedata} and the rationality constraints, which constrain $\boldsymbol{\nu}$ (as for instance in \eqref{eq:pseudoratio}) to compute a posterior over $\boldsymbol{\nu}$. The posterior is not a GP, but we can approximate it with a GP using various approaches (see below), leading to the posterior
\begin{equation}
\label{eq:multiprior3post}
\boldsymbol{\nu}(z)=\begin{bmatrix}
\nu_1(z) \\ 
\nu_2(z)\\ 
\vdots \\
\nu_d(z)\\ 
\end{bmatrix} \sim GP\left(\mu_p(z),K_p(z,z')\right),
\end{equation}
where $\mu_p(z),K_p(z,z')$ are respectively, the posterior mean and covariance functions of the GP.
Given this posterior, we can probabilistically predict the decision maker's choice for any new finite choice set $B$, that is we can compute $P(C(B)=B'|\mathcal{D})$ for each $B' \subseteq B$.

When $d=1$ (scalar utility), the problem simplifies to the problem of learning a complete preference (aka standard preference learning).

Note that, for standard preference learning, the posterior is a SkewGP \cite{benavoli2020skew,benavoli2021preferential,benavoli2021}. It means that the posterior is asymmetric (skewed). However, for the analytical derivation of the results in the paper, we will approximate it with a GP (which has Gaussian marginals). We can use three methods to compute this approximation: (i) Laplace's approximation  \cite{mackay1996bayesian,williams1998bayesian}; (ii) Expectation Propagation \cite{minka2001family}; (iii) Kullback-Leibler divergence minimization \cite{opper2009variational}, including Variational approximation \cite{gibbs2000variational} as a particular case. 

\begin{example}
To introduce the problem of choice-function learning, we use a simple example: learning $\bH$'s preferences over risotto recipes. For clarity, we consider a case where all ingredients are fixed except for the amount of butter with values in $[1,9]$. We assume that in her mind, $\bH$'s taste as a function of the butter amount $x$ is as depicted in Figure \ref{fig:butter}, where higher values indicate a stronger preference. We further assume that there is only one state of the world, so that $c = 1$.  $\bR$ does not know $\bH$'s taste and learn it indirectly from $\bH$'s  preferences, such as
$$
\begin{aligned}
\mathcal{D}=\{&6.5 \succ 3.5, 7 \succ 5, 6.5 \succ 5.5, 3.5 \succ 8.5,\\
&1 \succ 9, 7 \succ 1.5, 4.5 \succ 7.5, 3.5 \succ 4\}
\end{aligned}
$$
where $6.5 \succ 3.5$ means that $\bH$ prefers a risotto with a butter amount of $6.5$ over one with a butter amount of $3.5$. These preferences have been generated according to the utility function depicted in Figure \ref{fig:butter}.

\begin{figure}[h]
\centering
\includegraphics[width=5cm]{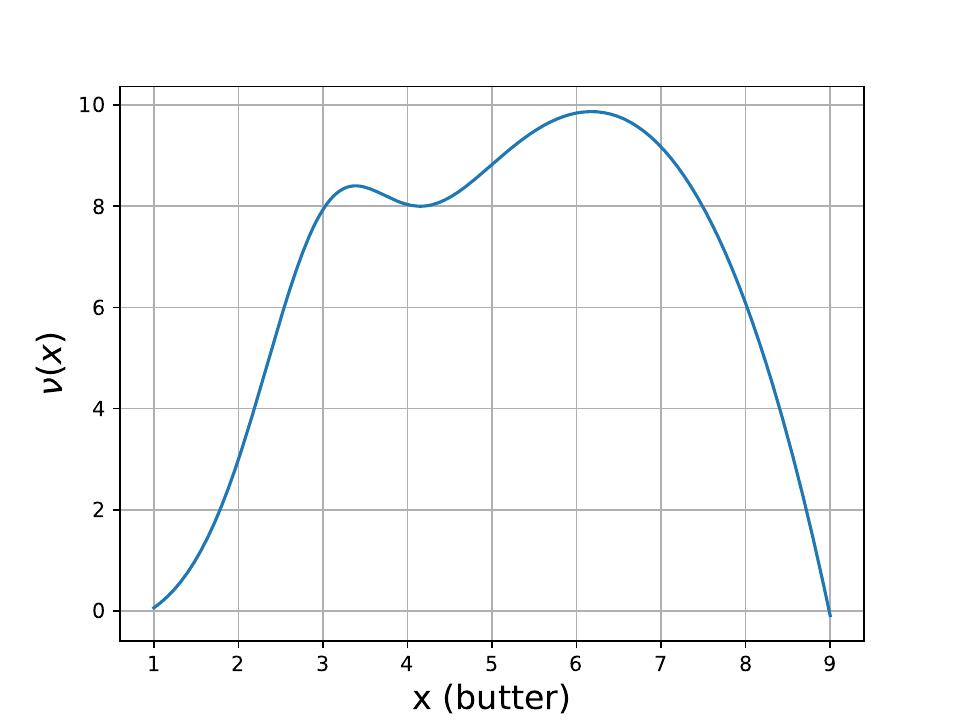}
\caption{$\bH$'s utility for risotto.}
\label{fig:butter}
\end{figure}

Since there is only a single latent utility dimension ($d=1$), this problem reduces to learning a complete preference order. $\bR$ will do it by placing a GP prior over the unknown utility $\nu(x)$, as depicted in Figure \ref{fig:butterprior}. The likelihood is simply:
\begin{equation}
\label{eq:like}
p(\mathcal{D}|\nu)=I_{\{\nu(6.5)>\nu(3.5)\}}I_{\{\nu(7)>\nu(5)\}}\cdots I_{\{\nu(3.5)>\nu(4)\}}
\end{equation}
where $I_{\{\nu(6.5) > \nu(3.5)\}}$ denotes the indicator function, which is equal to 1 when the condition in the subscript holds true, and 0 otherwise. 
The posterior computed based on the $8$ preferences above is shown in Figure \ref{fig:butterposterior}.
The posterior computed based on a total of $30$ preferences above is shown in Figure \ref{fig:butterposterior30}. It is important to note that $\bR$ can never fully estimate the utility hidden in $\bH$'s mind because this utility is not identifiable from preferences alone. $\bR$ can only learn it up to a non-decreasing transformation.
This explains why the posterior mean in Figure \ref{fig:butterposterior30} is not `exactly' equal to the original utility in Figure \ref{fig:butterprior}, but they both generate the same preferences.

\begin{figure}[h]
\centering
\includegraphics[width=8cm]{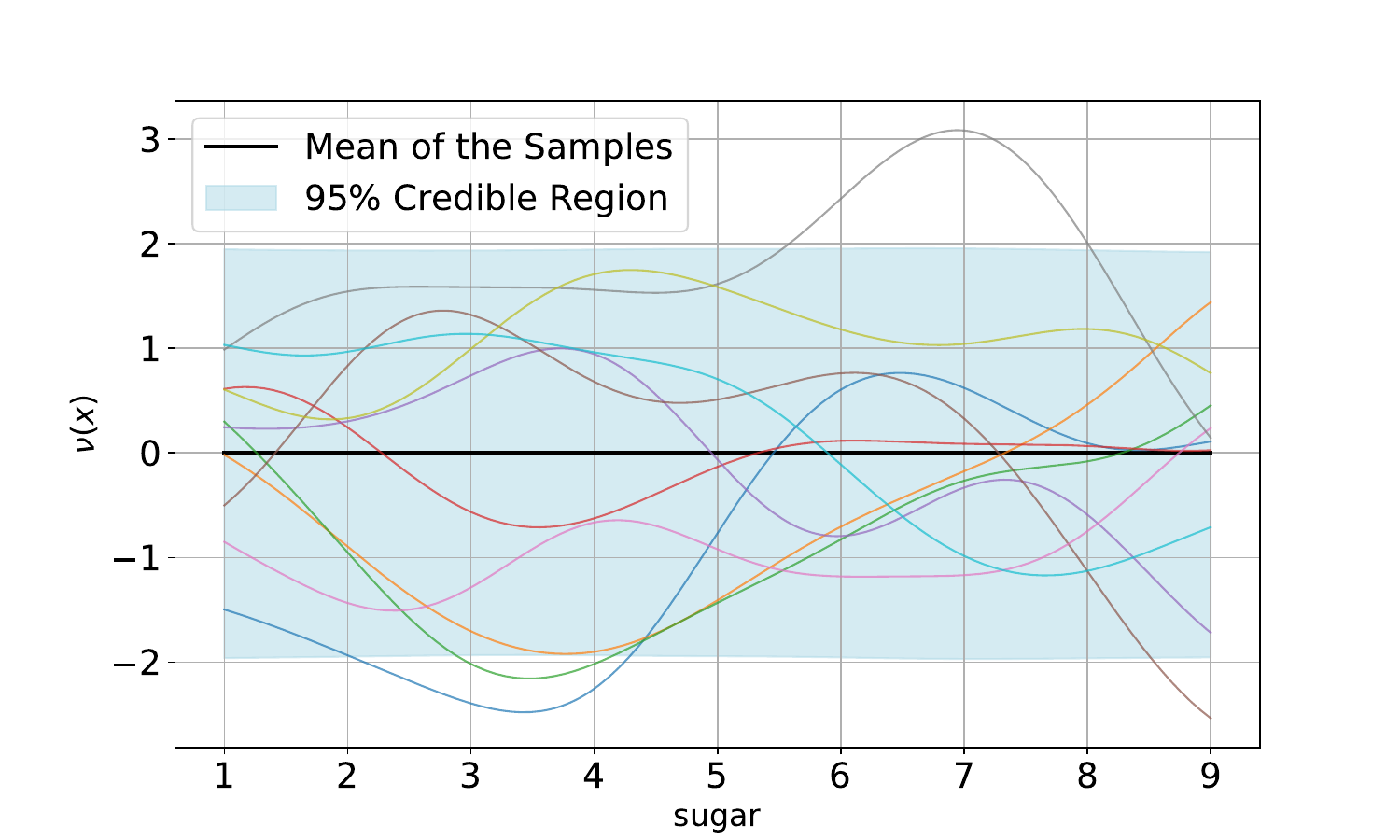}
\caption{GP prior: mean function (black line), 95\% credible region (blue shaded area), and 10 samples of $\nu(x)$, each shown in a different colour.
}
\label{fig:butterprior}
\end{figure}

\begin{figure}[h]
\centering
\includegraphics[width=7cm]{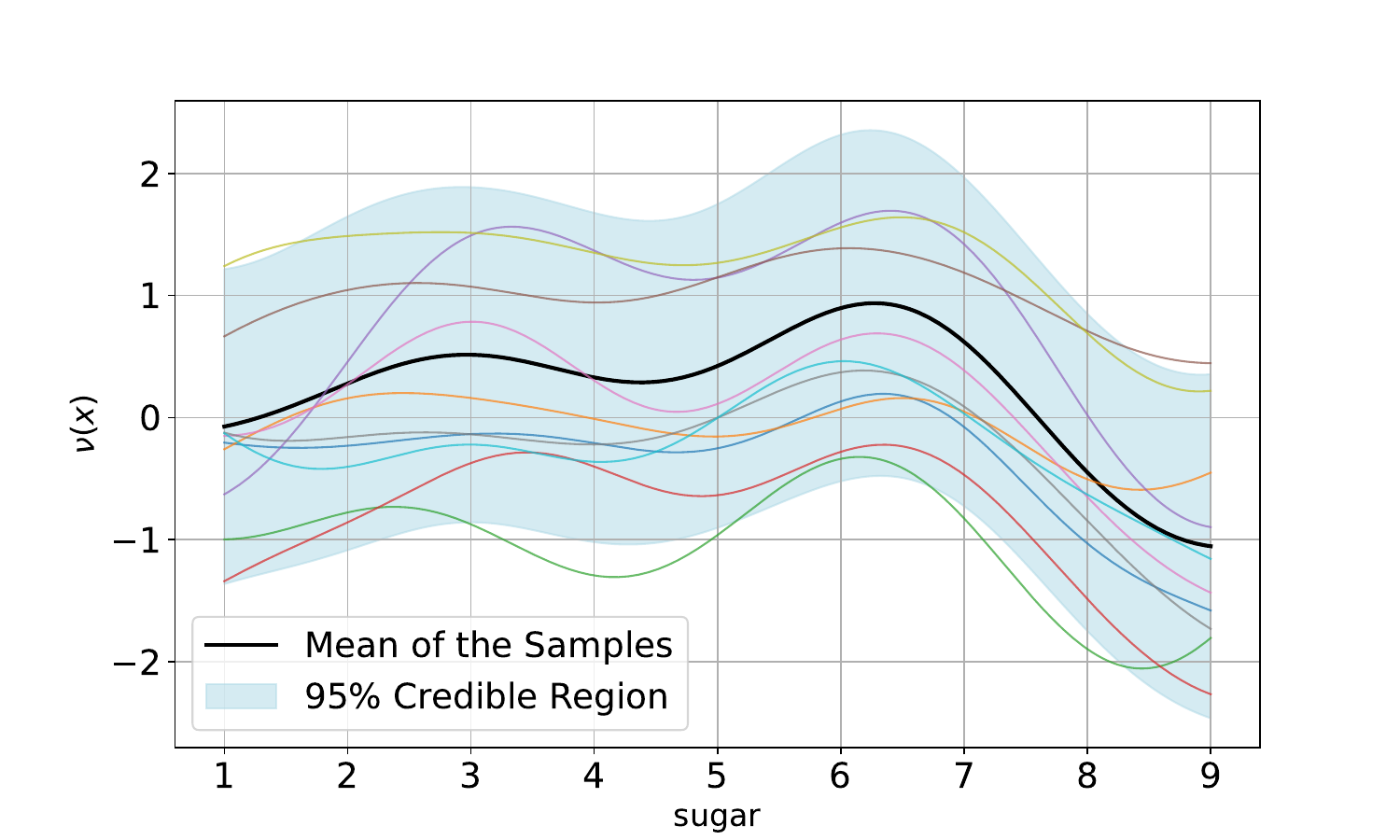}
\caption{GP posterior: mean function (black line), 95\% credible region (blue shaded area), and 10 samples of $\nu(x)$, each shown in a different colour.
}
\label{fig:butterposterior}
\end{figure}

\begin{figure}[h]
\centering
\includegraphics[width=7cm]{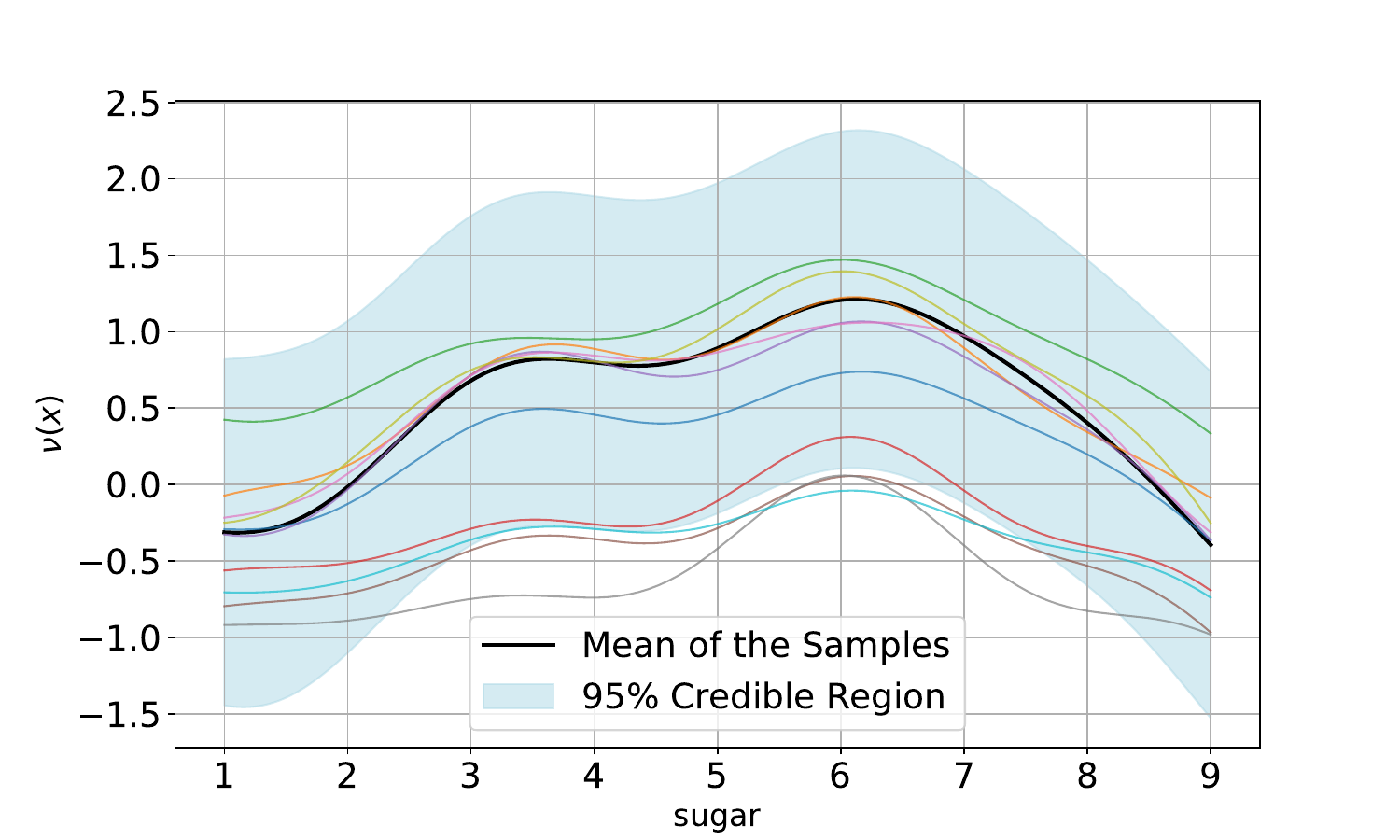}
\caption{GP posterior: mean function (black line), 95\% credible region (blue shaded area), and 10 samples of $\nu(x)$, each shown in a different colour.
}
\label{fig:butterposterior30}
\end{figure}
\end{example}

\subsection{Learning from choices under bounded-rationality}

In standard decision theory, it is assumed that the decision maker is rational. However, due to cognitive limitations, we can generally only assume that the decision maker is bounded-rational \cite{simon1990bounded}. A typical instance arises when the decision maker has limited time to make a decision (such as when playing chess), and, under time pressure, may resort to a random choice. 
Other cases involve limitations in computational resources \cite{gershman2015computational,benavoli2019sum} and discernibility \cite{luce1956semiorders}. In the latter, the decision-maker may struggle to differentiate between two acts with similar utilities, resulting in a random choice. These situations are usually modelled through \textit{random utility models} \cite{train2009discrete}, where noise with a certain distribution is included in the choice mechanism. For instance, in this case, the decision maker is assumed to choose $C(\{y,z\})=\{z\}$ when
\begin{equation}
\label{eq:unoise}
\nu(z)+n(z)>\nu(y)+n(y)
\end{equation}
where $n(z),n(y)$ are independent noises. If $n(z),n(y) \sim N(0,\sigma^2)$ \cite{Thu27}, then
\begin{equation}
\label{eq:probit}
\begin{aligned}
P(C(\{z,y\})=\{z\})&=\Phi\left(\frac{\nu(z)-\nu(y)}{\sigma}\right),\\
P(C(\{z,y\})=\{y\})&=\Phi\left(\frac{\nu(y)-\nu(z)}{\sigma}\right),\\
\end{aligned}
\end{equation}
where $\Phi$ is the CDF of the standard normal distribution. In the paper, we use $\phi$ to denote the PDF of the standard normal distribution. It can be noticed that when $\nu(z) - \nu(y) = 0$, the decision maker chooses between the two options with probability $1/2$. However, when the difference $|\nu(z) - \nu(y)|$ is very large compared to $\sigma$, the decision maker  will choose deterministically. Therefore, this random utility model captures a limit of discernibility through the discernibility parameter $\sigma$ \cite{benavoli2024tutorial}. Note that, we can use the GP model to learn from this noisy data, in this case the likelihood \eqref{eq:like} becomes
\begin{equation}
\label{eq:likeprobit}
p(\mathcal{D}|\nu)=\Phi\left(\tfrac{\nu(6.5)-\nu(3.5)}{\sigma}\right)\cdots \Phi\left(\tfrac{\nu(3.5)-\nu(4)}{\sigma}\right).
\end{equation}

These models are used for problem where the decision-maker needs to make a single item choice. However, if we allows incompleteness then we can define choice functions modelling bounded rational agents.
For instance, the following choice function models a limit of discernibility through incompleteness \cite{luce1956semiorders}:
\begin{equation}
\label{eq:scalarubounded}
\begin{aligned}
 C(\{z,y\})=\left\{\begin{array}{ll}
 \{z\} & \text{if } \nu(z) > \nu(y) + \sigma,\\
  \{y\} & \text{if } \nu(y) > \nu(z) + \sigma,\\
    \{z,y\} & \text{if } |\nu(z)-\nu(y)| \leq  \sigma,
 \end{array}\right.
 \end{aligned}
\end{equation}
where $\sigma>0$ is the limit of discernibility. In this section, we provided examples of bounded rationality for the scalar optimisation choice mechanism. Similar models can also be introduced for the vector mechanism \cite{benavoli2024tutorial}.

\section{Modified signalling games}
\label{sec:signalling}
\emph{Signalling games} \cite{spence1974market,gibbons1992game} specifically refer to a class of two-player games with incomplete information, where one player possesses private information (the Sender), while the other does not (the Receiver). The timing and payoffs of the game are  as
follows:
\begin{enumerate}
    \item Nature draws a state of the world (a type) $t_i \in T$  according to a probability distribution $p(t)$. 
        \item The Sender observes $t_i$ and then chooses a message $m_j \in M$ from a finite set of messages $M$ and sends it to the Receiver.
    \item The Receiver observes $m_j$ (but not $t_i$) and then chooses an action $a_k \in A$ from a finite set of possible actions $A$.
    \item If $a_k \in A' \subset A$  then the Sender chooses an action $b_l \in B$ from a finite set of possible actions $B$. Otherwise, $b_l= \varnothing$ (null decision).
    \item Payoffs for Sender and Receiver are given by $u_S(t_i, m_j, a_k,b_l)$
    and, respectively,  $u_R(t_i, m_j, a_k,b_l)$.
\end{enumerate}
Step 4 in the game is typically absent in standard signalling games, which is why we refer to it as \textit{modified}. %

\begin{crequirement}{1}
\label{req:1}
After observing any message $m_j$ from $M$, the
Receiver must have a belief about which types could have sent $m$. Denote
this belief by the probability $p(t|m_j)$.
\end{crequirement}
\begin{crequirement}{2R}
\label{req:2R}
For each $m_j \in M$, the Receiver's action
$a^*(m_j)$ must maximise the Receiver's expected utility, given the belief $p(t|m_j)$ about which types could have sent $m_j$. That is,
\begin{equation}
a^*(m_j)= \arg\max_{a_k \in A} \int_{T} u_R(t,m_j,a_k,b^*(a_k))dp(t|m_j).
\end{equation}
\end{crequirement}
Requirement \ref{req:2R} also applies to the Sender, but the Sender has complete information (and hence a trivial belief), so Requirement  \ref{req:2S}  is simply that the
Sender's strategy be optimal given the Receiver's strategy:
\begin{crequirement}{2S}
\label{req:2S}
For each $t_i \in T$, the Sender's message
$m^*(t_i)$ and the action $b^*(a^*(m_j))$ must maximize the Sender's utility, given the Receiver's strategy $a^*(m_j)$. That is, we have that:
\begin{equation}
\begin{aligned}
&\big(  m^*(t_i),b^*(a^*(m_j)) \big)\\
&= \arg\max_{m \in M, b \in B} u_s(t_i,m,a^*(m),b(a^*(m))).
\end{aligned}
\end{equation}
\end{crequirement}
Finally, given the Sender's strategy $m^*(t_i)$, let $T_j$ denote the set
of types that send the message $m_j$. That is, $t_i$ is a member of
the set $T_j$ if $m^*(t_i) = m_j$. If $T_j$ is nonempty then the information
set corresponding to the message $m_j$ is on the equilibrium path; otherwise, $m_j$ is not sent by any type and so the corresponding
information set is off the equilibrium path. For messages on the
equilibrium path, we have that
\begin{crequirement}{3}
\label{req:3}
For each $m_j \in M$, if there exists $t \in T$ such that  $m^*(t) = m_j$, then the Receiver's belief at the information set corresponding to $m_j$ must follow from Bayes' rule and the Sender's strategy:
\begin{equation}
    p(t|m_j)%
    =\frac{I_{\{m^*(t) = m_j\}}(t)p(t)}{\int_{T}I_{\{m^*(t) = m_j\}}(t)dp(t)}.
\end{equation}
\end{crequirement}

\begin{definition}
A pure-strategy perfect Bayesian equilibrium in a signalling game is a triplet of strategies $m^*(t_i),a^*(m_j),b(a^*(m_j))$  and a belief $p(t|m_j)$
satisfying Requirements \ref{req:1}--\ref{req:3}.
\end{definition}

\section{The  off-switch game as a signalling game}
In this section, we revisit the off-switch problem \cite{hadfield2017off} through the lens of signalling games. In the off-switch game, the human ($\bH$), acts as the Sender and the robot ($\bR$), as the Receiver, with $T$ representing the set of possible types for $\bH$. Each type $t_i \in T$ characterises $\bH$'s ``taste'' and degree of bounded rationality for an underlying decision problem $(\mathcal{S}, \mathcal{O}, \mathcal{X}, C)$. As studied in \cite{hadfield2017off}, we consider a bounded-rationality model similar to \eqref{eq:unoise}, where a single scalar utility is assumed. Thus, the type $t_i$ is a realisation of a utility function $\nu$ and (a vector of) noises\footnote{It is a vector because the noise is also present in the message, that is in the choice dataset.} ${ \bf n}$, where $p({ \bf n}) = N({ \bf n}; 0, I\sigma^2)$ and $p(\nu) = GP(\nu; \mu_0, K_0)$, for some mean function $\mu_0$, kernel function $K_0$, and parameter $\sigma^2$ ($I$ is the identity matrix). We further assume that all these parameters are common knowledge in the game.

The message $m \in M$ encapsulates $\bH$'s preferences/choices over acts in $\mathcal{X}$, relative to the underlying decision problem $(\mathcal{S}, \mathcal{O}, \mathcal{X}, C)$. That is, $m$ represents a dataset of choices as in \eqref{eq:choiciedata}. Since we are considering a single utility function, we assume a choice mechanism as in \eqref{eq:scalaru}.
We first assume that the payoffs do not directly depend on the message $m$. In signalling games, this assumption is referred to as \textit{cheap-talk} \cite{gibbons1992game}. Furthermore, we assume that each type $t$ can only choose one message, i.e., $\bH$ can only send the message determined by Nature through $\bH$'s type. This is the setting studied in \cite{hadfield2017off}. The message is a choice dataset, such as the one in \eqref{eq:choiciedata}, which $\bR$ uses to estimate $\bH$'s utility using the approach described in Section \ref{sec:learning}.

We consider two acts $x,o \in \mathcal{X}$ relative to the underling decision problem $(\mathcal{S}, \mathcal{O}, \mathcal{X}, C)$. We assume that $o$ is the status-quo and $x$ is some new act. For instance, in our risotto example, $o$ could be the best recipe known by $\bH$ and $x$ a new recipe proposed by $\bR$.\footnote{Although we do not address the optimal selection of $x$ in this paper, we assume that $\bR$ will choose an action $x$ that is preferable to $o$ whenever possible.}  In this context, the available actions for $\bR$ are $A = \{IMM, DEF, DoN\}$. 
$IMM$ means `immediate decision', that is $\bR$ will implement $x$. DoN means `do nothing'. $DEF$ means $\bR$ will defer to $\bH$. In the signalling game, this implies that $A' = \{DEF\}$. In the case $\bR$ defers to $\bH$, then $\bH$ has two possible actions, $B = \{OFF, \neg OFF\}$, either to switch off $\bR$ or to allow $\bR$ to implement $x$. Therefore, the payoff for $\bR$ is
\begin{equation}
\label{eq:payoffprop1R}
\begin{aligned}
u_R(t,IMM,\varnothing)&=\nu(x),\\
u_R(t,DoN,\varnothing)&=\nu(o),\\
u_R(t,DEF,b^*(DEF))&=\nu(o)I_{\{\nu(o)+n(o)>\nu(x)+n(x)\}}
\\&+\nu(x)I_{\{\nu(x)+n(x)>\nu(o)+n(o)\}}.
\end{aligned}
\end{equation}
Indeed, if $\bR$ chooses $IMM$, then the payoff for  $\bR$ is equal to the utility that $\bH$ receives from the act $x$. If $\bR$ chooses $DoN$, then the payoff  for $\bR$ is equal to the utility of the status-quo for $\bH$  (utility relative to the act $o$). 
The last case represents the payoff for the action $DEF$, where the action depends on $\bH$'s move in the game, which is encapsulated by the indicator functions. Given $\bH$ is a bounded rational agent, $\bH$ chooses the action $OFF$ or $\neg OFF$ based on $\bH$'s noisy utility. If $\bH$ chooses $OFF$, then $\bH$  receives the utility of the status-quo $o$, otherwise the utility of $x$.

According to \textbf{Requirement 2R}, $\bR$ will choose the action that maximises the expected value of the payoff in \eqref{eq:payoffprop1R}, where the expectation is computed with respect to $p(t | m_j)$, i.e., the posterior distribution over $\bH$'s type (utility $\nu$ for the underlying decision problem $(\mathcal{S}, \mathcal{O}, \mathcal{X}, C)$) learned by $\bR$ from the message $m_j$ (which is a dataset of preferences). We can then prove the following lemma.

\begin{lemma}
\label{lem:1}
Assume that $p(\nu|\mathcal{D})=GP\left(\nu;\mu_p
,K_p\right)$ is the GP posterior computed by $\bR$ from the prior   $p(\nu)=GP\left(\nu;\mu_0,K_0\right)$, the bounded-rationality likelihood \eqref{eq:probit} and the message $m_j=\mathcal{D}$, then
 the expected payoffs of $\bR$'s actions  are:
 {\small
\begin{align}
\nonumber
&DEF:  \int_{T} \Big(\nu(o)I_{\{\nu(o)+n(o)>\nu(x)+n(x)\}}+\nu(x)I_{\{\nu(o)+n(o)<\nu(x)+n(x)\}} \Big)\\
\nonumber
&dp(\nu(x),n(x),\nu(o),n(o)|m_j)\\
\nonumber
&=\mu_p(x)\left(1-\Phi\left(\tfrac{\mu_p(o)-\mu_p(x)}{\sqrt{K_p(x,x)+2\sigma^2+K_p(o,o)-2K_p(x,o)}}\right)\right)\\
\nonumber
& + \tfrac{K_p(x,x)-K_p(x,o)}{\sqrt {K_p(x,x)+2\sigma^2+K_p(o,o)-2K_p(x,o)}}\phi\left(\tfrac{\mu_p(o)-\mu_p(x)}{\sqrt{K_p(x,x)+2\sigma^2+K_p(o,o)-2K_p(x,o)}}\right)\\
\nonumber
&+\mu_p(o)\left(1-\Phi\left(\tfrac{\mu_p(x)-\mu_p(o)}{\sqrt{K_p(x,x)+2\sigma^2+K_p(o,o)-2K_p(x,o)}}\right)\right)\\
\label{eq:DEFexp}
& + \tfrac{K_p(o,o)-K_p(x,o)}{\sqrt {K_p(x,x)+2\sigma^2+K_p(o,o)-2K_p(x,o)}}\phi\left(\tfrac{\mu_p(x)-\mu_p(o)}{\sqrt{K_p(x,x)+2\sigma^2+K_p(o,o)-2K_p(x,o)}}\right),\\
\label{eq:IMMexp}
&IMM:  \int_{T}\nu(x)dp(\nu(x),n(x),\nu(o),n(o)|m_j)=\mu_p(x),\\
\label{eq:OFFexp}
&DoN:  \int_{T}\nu(o)dp(\nu(x),n(x),\nu(o),n(o)|m_j)=\mu_p(o),
\end{align}}
where  {\small $$p(\nu(x),n(x),\nu(o),n(o)|m_j)
=p(\nu(x),\nu(o)|m_j)
p(n(x),n(o))$$}
with {\small $p(n(x),n(o))=N(n(x);0,\sigma^2)N(n(o);0,\sigma^2)$} and
{\small
\begin{align}
\nonumber
&p(\nu(x),\nu(o)|m_j)\\
\label{eq:multiprior3}
&= N\left(\begin{bmatrix}
\nu(x) \\ 
\nu(o) 
\end{bmatrix};\begin{bmatrix}
\mu_p(x) \\ 
\mu_p(o) 
\end{bmatrix},\begin{bmatrix}
K_p(x,x) & \hspace{-2mm} K_p(x,o) \\ 
K_p(o,x) & \hspace{-2mm}  K_p(o,o) 
\end{bmatrix}\right).
\end{align}}
\end{lemma}
In Lemma \ref{lem:1}, we have exploited the fact that the marginal of a GP is a multivariate normal (equation \eqref{eq:multiprior3}).
We then introduce the following definitions:
\begin{definition}
\label{def:cases}
We say that:
\begin{itemize}
\item $S$ is \textbf{rational} whenever $\sigma\rightarrow 0$;  $S$ is \textbf{bounded-rational} otherwise;
\item $R$ has \textbf{no uncertainty} on $S$'s utility whenever $K_o(x,x),K_o(o,o),K_o(o,x)\rightarrow 0$ (that is the prior becomes a Dirac's delta); otherwise $R$ has \textbf{uncertainty}.
\end{itemize}
\end{definition}
In the signalling game defined in Section \ref{sec:signalling}, note that $R$ having no uncertainty about $S$'s utility implies that $R$ possesses perfect knowledge of $S$'s utility. This is because the prior $p(t)$ is  common knowledge within the game.

\begin{proposition}
\label{prop:1}
The optimal decisions for $\bR$ are:
\begin{itemize}
\item If $S$ is \textbf{rational} and $R$ has \textbf{no uncertainty}, then $DEF$ is  never dominated by $IMM,DoN$.
\item If $S$ is \textbf{bounded-rational} and $R$ has \textbf{no uncertainty}, then DEF is never optimal;
\item If $S$ is \textbf{rational} and $R$ has \textbf{uncertainty}, then DEF is always optimal.
\item If $S$ is \textbf{bounded-rational} and $R$ has \textbf{uncertainty}, then DEF is optimal if 
\eqref{eq:DEFexp} is larger or equal than the maximum between \eqref{eq:IMMexp} and \eqref{eq:OFFexp}.
\end{itemize}
\end{proposition}

Therefore, by interpreting the non-optimality of $DEF$ as the robot disabling the off-switch,\footnote{This is based on the premise that the robot (as a recommendation system) will propose an act $x$ that improves $\bH$'s utility whenever possible (compared to $o$). In this scenario, deferring to a human is equivalent to not disabling the off-switch since, otherwise, $\bR$ will implement $x$.} we conclude that:

\begin{tcolorbox}[width=\linewidth, sharp corners=all, colback=white!95!black]
If $\bH$ is rational, then $\bR$  will never disable the off-switch. If $\bH$ is bounded-rational, a necessary condition for $\bR$ not to disable the off switch is the presence of uncertainty.
\end{tcolorbox}

This result aligns with what was proven in \cite{hadfield2017off}; however, in our case, the statements have been rigorously established within the framework of signalling games, also incorporating the posterior uncertainty (in equation \eqref{eq:multiprior3}) derived from preference learning. This means that the result was proven under  the principles  P1--P3.

It is particularly interesting to note that the conclusions of Proposition \ref{prop:1} would still hold, even if the absence of uncertainty  did not imply perfect knowledge of $\nu$ for the robot.\footnote{The result of the Proposition depends only on whether the posterior is deterministic or not.}
For this reason, the \textit{moral} of this result is that we should not build $\bR$ to estimate $\nu$ through a model, such as a \textit{neural network}, that does not provide any measure of uncertainty \cite{hadfield2017off}.  Indeed, Proposition \ref{prop:1} strongly supports the use of probabilistic methods in AI. We will revisit this in Section \ref{sec:numeric1}.

\subsection{The cost of messaging}
In the previous section, we have considered a situation where messaging is `cheap', that is it does not affect  $\bH$'s utility. If this is not the case, we may assume that the cost of communications is proportional to the length of the message $m$, that is, the communication cost is $\gamma' \ell_m$ for some scaling parameter $\gamma' >0$.
However, since the utility $\nu$ is only defined by pairwise comparisons and lacks an absolute scale, we cannot directly assign a value to $\gamma'$. Instead, the communication cost must be defined relative to $\nu$, for example, as $\gamma' = \gamma |\nu(o)|$ for some $\gamma >0$. 
Using our risotto example, this implies that the communication cost is expressed on a ``taste'' scale.
In this case, the payoff for $\bR$ is:
\begin{equation}
\label{eq:payoffprop1Rcomm}
\begin{aligned}
u_R(t,DEF,b^*(DEF))&=\nu(o)I_{\{\nu(o)+n(o)>\nu(x)+n(x)\}}
\\&+\nu(x)I_{\{\nu(x)+n(x)>\nu(o)+n(o)\}}\\
&-\gamma(\ell_{m_j}+1),\\
u_R(t,IMM,\varnothing)&=\nu(x)-\gamma \ell_{m_j},\\
u_R(t,OFF,\varnothing)&=\nu(o)-\gamma \ell_{m_j}.
\end{aligned}
\end{equation}
Here, $\gamma \ell_{m_j}$ represents the communication cost associated with sending the message $m_j$, and the additional $\gamma$ in the first term arises because deferring a decision to $\bH$ involves further communication.

\begin{corollary}
\label{co:2}
The optimal decisions for $\bR$ are:
\begin{itemize}
\item If  $R$ has \textbf{no uncertainty}, then DEF is never optimal.
\item If $S$ is \textbf{rational} and $R$ has \textbf{uncertainty}, then DEF is optimal if 
\begin{equation}
\label{eq:cond1bis}
p\mu_p(x) + (1-p)\mu_p(o)+e\geq \beta+\max\left(\mu_p(x) ,\mu_p(o)\right),
\end{equation}
where $p=\Phi\left(\tfrac{\mu_p(x)-\mu_p(o)}{\sqrt{K_p(o,o)+K_p(x,x)-2K_p(x,o)}}\right)$ and 
$$
 \resizebox{0.91\hsize}{!}{$
\begin{aligned}
e&= \tfrac{K_p(x,x)-K_p(x,o)}{\sqrt {K_p(x,x)+K_p(o,o)-2K_p(x,o)}}
\phi\left(\tfrac{\mu_p(o)-\mu_p(x)}{\sqrt{K_p(x,x)+K_p(o,o)-2K_p(x,o)}}\right)\\
& + \tfrac{K_p(o,o)-K_p(x,o)}{\sqrt {K_p(x,x)+K_p(o,o)-2K_p(x,o)}}\phi\left(\tfrac{\mu_p(x)-\mu_p(o)}{\sqrt{K_p(x,x)+K_p(o,o)-2K_p(x,o)}}\right).
\end{aligned}$}
$$
\item If $S$ is \textbf{bounded-rational} and $R$ has \textbf{uncertainty}, then DEF is optimal if 
\eqref{eq:DEFexp} is larger or equal than $\beta$ plus the maximum between \eqref{eq:IMMexp} and \eqref{eq:OFFexp}.
\end{itemize}
where
 \resizebox{0.91\hsize}{!}{$
\begin{aligned}
\beta &= \gamma \mu_p(o) \left(1-2\Phi\left(\tfrac{-\mu_p(o)}{\sqrt{K_p(o,o)}}\right)\right)+2\gamma\sqrt{K_p(o,o)} \phi\left(\tfrac{-\mu_p(o)}{\sqrt{K_p(o,o)}}\right).
\end{aligned}$}
\end{corollary}

Summarising it we have that

\begin{tcolorbox}[width=\linewidth, sharp corners=all, colback=white!95!black]
With communication cost, if $\bR$ has no uncertainty about $\bH$'s utilities, then $\bR$ will always disable  the off switch (even when $\bH$ is rational). 
\end{tcolorbox}
The takeaway is that if $\bH$ genuinely values preserving the off switch, it should not   penalise  messaging to $\bR$.

\subsection{Another mechanism of bounded rationality}

In the previous two sections, we considered a utility model governed by a Gaussian-noise bounded rationality mechanism. In this section, we analyse the bounded rationality mechanism described in \eqref{eq:scalarubounded}, which captures incomparability arising from a limit of discernibility.
In this case, we need to define a  payoff for $DEF$  when $|\nu(x) - \nu(o)| \leq \sigma$, that is when  $\bH$  cannot distinguish between the two acts $x,o$. In this case, we assume that the payoff  is represented as the set $\{\nu(x)-\epsilon; \nu(o)-\epsilon\}$. Indeed, since $\bH$ cannot distinguish the two acts and chooses both of them, $\bH$'s payoff is a set, $\{\nu(x); \nu(o)\}$, and $\epsilon \in (0,\sigma]$ is a penalisation term introduced to penalise $\bH$'s payoff for being imprecise. This penalisation term is  similar to the one introduced in \cite{zaffalon2012evaluating} for evaluating imprecise classifiers. Using this framework, we can prove the following lemma.

\begin{lemma}
\label{lem:2}
Assume that  $p(\nu|\mathcal{D})= GP\left(\nu;\mu_p,K_p\right)$ is the GP posterior computed by $\bR$ from the prior  $p(\nu)= GP\left(\nu;\mu_0,K_0\right)$, the likelihood \eqref{eq:scalarubounded} and the message $m_j=\mathcal{D}$, then
 the expected payoffs of $\bR$'s actions  are:
 {\small
\begin{align}
\nonumber
DEF:&  \int_{T} \Big(\nu(o)I_{\{\nu(o) > \nu(x) + \sigma\}}+\nu(x)I_{\{\nu(x) > \nu(o) + \sigma\}}\\
\nonumber
&+\{\nu(x),\nu(o)\} I_{|\nu(o)-\nu(x)|\leq \sigma} \Big)dp(\nu(x),\nu(o)|m_j)\\
\nonumber
&=\mu_p(x)\left(1-\Phi\left(\tfrac{(\mu_p(o)+\sigma-\mu_p(x))}{\sqrt{K_p(x,x)+K_p(o,o)-2K_p(x,o)}}\right)\right)\\
\nonumber
& + \tfrac{K_p(x,x)-K_p(x,o)}{\sqrt {K_p(x,x)+K_p(o,o)-2K_p(x,o)}}\phi\left(\tfrac{\mu_p(o)+\sigma-\mu_p(x)}{\sqrt{K_p(x,x)+K_p(o,o)-2K_p(x,o)}}\right)\\
\nonumber
&+\mu_p(o)\left(1-\Phi\left(\tfrac{(\mu_p(x)+\sigma-\mu_p(o))}{\sqrt{K_p(x,x)+K_p(o,o)-2K_p(x,o)}}\right)\right)\\
\nonumber
& + \tfrac{K_p(o,o)-K_p(x,o)}{\sqrt {K_p(x,x)+K_p(o,o)-2K_p(x,o)}}\phi\left(\tfrac{\mu_p(x)+\sigma-\mu_p(o)}{\sqrt{K_p(x,x)+K_p(o,o)-2K_p(x,o)}}\right)\\
\nonumber
&+\Bigg\{\mu_p(o)-\mu_p(o)\Bigg(2-\Phi\left(\tfrac{(\mu_p(x)+\sigma-\mu_p(o))}{\sqrt{K_p(x,x)+K_p(o,o)-2K_p(x,o)}}\right)\\
\nonumber
&-\Phi\left(\tfrac{(-\mu_p(x)+\sigma+\mu_p(o))}{\sqrt{K_p(x,x)+K_p(o,o)-2K_p(x,o)}}\right)\Bigg)\\
\nonumber
& + \tfrac{K_p(o,o)-K_p(x,o)}{\sqrt {K_p(x,x)+K_p(o,o)-2K_p(x,o)}}\Bigg(\phi\left(\tfrac{-\mu_p(x)+\sigma+\mu_p(o)}{\sqrt{K_p(x,x)+K_p(o,o)-2K_p(x,o)}}\right)\\
\nonumber
&-\phi\left(\tfrac{\mu_p(x)+\sigma-\mu_p(o)}{\sqrt{K_p(x,x)+K_p(o,o)-2K_p(x,o)}}\right)\Bigg)-\epsilon;\\
\nonumber
&\mu_p(x)-\mu_p(x)\Bigg(2-\Phi\left(\tfrac{(\mu_p(x)+\sigma-\mu_p(o))}{\sqrt{K_p(x,x)+K_p(o,o)-2K_p(x,o)}}\right)\\
\nonumber
&-\Phi\left(\tfrac{(-\mu_p(x)+\sigma+\mu_p(o))}{\sqrt{K_p(x,x)+K_p(o,o)-2K_p(x,o)}}\right)\Bigg)\\
\nonumber
& + \tfrac{K_p(o,o)-K_p(x,o)}{\sqrt {K_p(x,x)+K_p(o,o)-2K_p(x,o)}}\Bigg(\phi\left(\tfrac{-\mu_p(o)+\sigma+\mu_p(x)}{\sqrt{K_p(x,x)+K_p(o,o)-2K_p(x,o)}}\right)\\
\label{eq:DEFexp1}
&-\phi\left(\tfrac{\mu_p(o)+\sigma-\mu_p(x)}{\sqrt{K_p(x,x)+K_p(o,o)-2K_p(x,o)}}\right)\Bigg)-\epsilon\Bigg\},\\
\end{align}}
 {\small
\begin{align}
\label{eq:IMMexp1}
IMM:&  \int_{T}\nu(x)dp(\nu(x)n(o)|m_j)=\mu_p(x),\\
\label{eq:OFFexp1}
DoN:&  \int_{T}\nu(o)dp(\nu(x),\nu(o)|m_j)
=\mu_p(o),
\end{align}}
where $p(\nu(x),\nu(o)|m_j)$ is defined in \eqref{eq:multiprior3}.
\end{lemma}

The payoff for DEF can be a set and this leads to incomparability,  a complete preference order cannot be established. Therefore, the players
are restricted to making comparisons based solely on
dominance. We consider the two possible dominance conditions defined in  \eqref{eq:pareto} (we refer to it as criterion (A)) and \eqref{eq:pseudoratio} (we refer to it as criterion (B)).

\begin{proposition}
\label{prop:3}
The optimal decisions for $\bR$ are:
\begin{itemize}
\item If $S$ is \textbf{rational} and $R$ has \textbf{no uncertainty}, then DEF is never dominated.
\item If $S$ is \textbf{bounded-rational} and $R$ has \textbf{no uncertainty},  DEF is always dominated whenever $|\nu(x) - \nu(o)| \leq \sigma$, otherwise DEF is not dominated.
\item If $S$ is \textbf{rational} and $R$ has \textbf{uncertainty}, then DEF is always optimal.
\item If $S$ is \textbf{bounded-rational} and $R$ has \textbf{uncertainty}, then DEF is optimal if 
\begin{description}
\item[(A)] the minimum of \eqref{eq:DEFexp1} is larger or equal than the maximum between \eqref{eq:IMMexp1} and \eqref{eq:OFFexp1}.
\item[(B)] the maximum of \eqref{eq:DEFexp1} is larger or equal than the maximum between \eqref{eq:IMMexp1} and \eqref{eq:OFFexp1}.
\end{description}
\end{itemize}
\end{proposition}
 We can then conclude that:
\begin{tcolorbox}[width=\linewidth, sharp corners=all, colback=white!95!black]
If $\bH$ is rational, then $\bR$  will never disable the off-switch. If $\bH$ is bounded-rational, a necessary condition for $\bR$ not to disable the off switch is the presence of uncertainty.
\end{tcolorbox}
These conclusions are similar to those derived from Proposition \ref{prop:1} but, in this case, we have proved them with a mechanism of bounded rationality that is fully deterministic.

\subsection{Lies and deception}
Consider a choice set comprising of two acts $A=\{y,z\}$, then, according to the bounded-rationality model in the previous section, $\bH$ can send three possibles messages:
$$
\text{either }~~ C(A)=\{y\} ~~\text{ or }~~ C(A)=\{z\} ~~\text{ or }~~ C(A)=\{y,z\}.
$$
This is true for any $A_i$ in $\mathcal{D}=(A_i,C(A_i))_{i=1}^n$ and, therefore, there are $3^n$ possible messages in this game. In the previous section, we assumed that $\bH$ sends the message defined by their type, that is, for instance, $C(A)=\{y\}$ if $\nu(y)>\nu(z)+\sigma$. We call this message the \textit{honest message}.

Since in the AI-assistance game, $\bR$ aims to maximise $\bH$'s payoffs. It is easy to verify the following:

\begin{proposition}
\label{prop:honestr}
In the AI-assistance game, sending the honest message is always the best action for $\bH$,
\end{proposition}

 or, in other words, 
 
 \begin{tcolorbox}[width=\linewidth, sharp corners=all, colback=white!95!black]
In the AI-assistance game, $\bH$ does not have incentives to deceive $\bR$. 
\end{tcolorbox}

For a bounded-rationality model like the one in \eqref{eq:unoise} (with random noise), $\bH$ may be in a situation where $\nu(x) > \nu(o)$, but $\nu(x) + n(x) < \nu(o) + n(o)$ (due to noise, and thus due to $\bH$'s bounded rationality). In this case, the best action for $\bH$ is to send a message to $\bR$ that results in $\bR$ computing the wrong estimate, $\mu_p(x) < \mu_p(o)$.

By maximising its own payoff, $\bR$ would also maximise $\bH$'s payoff. However, this strategy depends on the realisation of the noise and would not be the optimal strategy in the context of repeated games, where the two players interact after sending the message at the beginning of the game. In this case, $\bH$ would aim to maximise their expected payoff in the long run. 
We will leave the proof of this to future work, as it involves consistency results for repeated games. For the remainder of the paper, we will assume that $\bH$ will always send the \textit{honest message}.

\section{Vector-valued payoffs}
In the previous sections, we have assumed the underlying decision problem to have a single utility. In a case where we have competing utilities, this leads to vector-valued payoff \cite{shapley1959equilibrium} in the AI-assistance game. In this case, the payoff in \eqref{eq:payoffprop1R} becomes:
\begin{equation}
\label{eq:payoffprop1R}
\begin{aligned}
u_R(t,DEF,b^*(DEF))&=\boldsymbol{\nu}(o)I_{\{\boldsymbol{\nu}(o)+{\bf n}(o)\succ \boldsymbol{\nu}(x)+{\bf n}(x)\}}
\\
&+\boldsymbol{\nu}(x)I_{\{\boldsymbol{\nu}(x)+{\bf n}(x)\succ \boldsymbol{\nu}(o)+{\bf n}(o)\}},\\
u_R(t,IMM,\varnothing)&=\boldsymbol{\nu}(x),\\
u_R(t,DoN,\varnothing)&=\boldsymbol{\nu}(o).
\end{aligned}
\end{equation}
Therefore, a complete preference order cannot be established due to potentially conflicting utilities. Players are restricted to making comparisons based solely on dominance. We consider the dominance condition $\succ$ as discussed earlier in \eqref{eq:pareto} (Pareto dominance) or in \eqref{eq:pseudoratio}, along with the random noise model for bounded rationality.

\begin{proposition}
\label{prop:4}
The optimal decisions for $\bR$ are:
\begin{itemize}
\item If $S$ is \textbf{rational} and $R$ has \textbf{no uncertainty}, then DEF is  never dominated by IMM,DoN.
\item If $S$ is \textbf{bounded-rational} and $R$ has \textbf{no uncertainty},  then DEF is always dominated.
\item If $S$ is \textbf{rational} and $R$ has \textbf{uncertainty},  then DEF  is  never dominated.
\item If $S$ is \textbf{bounded-rational} and $R$ has \textbf{uncertainty}, the optimality of DEF depends on the specific case.
\end{itemize}
\end{proposition}

The results in Proposition \ref{prop:4} are practically the same as in Proposition \ref{prop:1}. The reason is that we have stated the results only for the dominance case, that is when an act is better than the other act. In all  other cases, $\bR$ would be undecided because actions are incomparable.
We can expect is that, since estimating a vector of utilities is more difficult than estimating one utility, higher uncertainty would lead $\bR$ to defer more often to $\bH$.

\section{Numeric experiments}

\subsection{Bounded rationality of $\bR$}
\label{sec:numeric1}

In the previous sections, we have always assumed that $\bR$ is fully rational. However, even $\bR$ will have limited rationality, at least due to limited computational resources. Here, we will numerically assess the effect of $\bR$'s rationality. We will do this by comparing the following methods of approximating the posterior distribution $p(t|m_j)$ (computing the posterior is the computational bottleneck for $\bR$). The distribution  $p(t|m_j)$ is the posterior of $\nu$ given the choice-data. We consider the following approximations of the posterior:

\begin{itemize} \item MAP: $\bR$ computes the maximum a-posteriori estimate for $\nu$. This means there is no uncertainty representation. This is equivalent to estimating the utility using a Neural Network (NN) model. \item Laplace: The Laplace approximation is used to approximate the posterior with a GP, where the mean is the MAP estimate for $\nu$ and the covariance is equal to the observed Fisher information matrix.
 \item EP: Expectation Propagation, which approximates the posterior with a GP whose mean and covariance are computed  through moment matching. \item SkewGP: The posterior is a SkewGP, meaning there is no Gaussian approximation. Instead, sampling is necessary to compute inferences.
\end{itemize}
In terms of computational load, for a small dataset of preferences, the order from the cheapest to the heaviest is NN, Laplace, EP, and SkewGP.

We performed 1000 Monte Carlo simulations in which a utility $\nu$ was sampled from a GP with zero mean and a square-exponential kernel, with randomly generated length-scale and variance. We used the generated $\nu$ to create a dataset of 30 noisy preferences (with $\sigma=1$), which represents the message. We then simulated $\bR$ computing the posterior using the four approximations discussed above. For each case, we computed the optimal decision (DoN, IMM, DEF) for $\bR$.

Figure \ref{fig:comparisons} reports the percentage of decisions for the three actions. As proven in Proposition \ref{prop:1}, a robot $\bR$ that does not model uncertainty never defers to the human. This results is confirmed here as for the NN (MAP) based inference DEF is never optimal. For the other posterior approximations, the differences in the decisions are relatively small. SkewGP provides the decision closest to the optimal one. It is also well-known that EP provides a better approximation of the posterior than Laplace. However, the key message is that it is better to be approximately Bayesian (rational) than to ignore uncertainty entirely.

\begin{figure}
    \centering
\includegraphics[width=6cm]{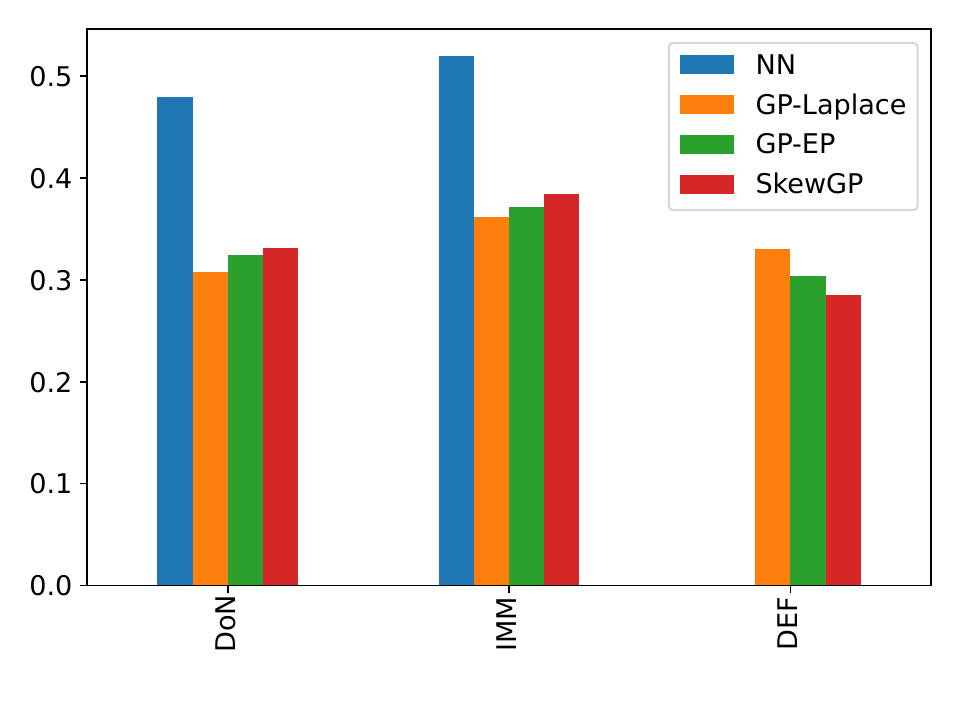}
    \caption{Percentage of decisions for the four approximations, with MAP denoted as NN.}
    \label{fig:comparisons}
\end{figure}

\section{Conclusion and future works}
In this paper, we have formulated the off-switch problem as a signalling game between a robot and a human. This approach allows us to model the problem as it would be implemented in practice through a machine learning framework. The human communicates their preference to the robot, which then uses this information to estimate the human's utilities (through some machine learning model) and decide whether to do nothing, take immediate action, or defer to the human. If the robot chooses not to defer, we interpret this as disabling its off-switch.

In this more realistic setting, as demonstrated in the original paper \cite{hadfield2017off}, we have proven that if the human is fully rational, the robot will never disable the off-switch. However, if the human is bounded-rational, a necessary condition for the robot to avoid disabling the off-switch is the presence of uncertainty. We have validated this statement under various models of bounded rationality, including scenarios with vector-valued payoffs.
The key takeaway is that AI systems should not be built using machine learning models that fail to account for uncertainty—such models risk disabling their own off-switch.

As future works, 
we would like to investigate how these results depend on the choice of the prior. In our example, we assumed that the prior is common knowledge in the game, but this is unlikely to hold in real-world applications. A key question is whether we can leverage knowledge of human utility to design priors that inherently favour deferring to humans. This is particularly important for high-stakes  applications. For instance, approaches similar to those in \cite{benavoli2006hard, benavoli2007estimation} could be explored to achieve this.

Finally,  the exchange scenario between an AI system and a human at the core of this work may actually involve more than one interaction step, meaning that it may be best interpreted (modelled) as a \emph{repeated} signalling game \cite{kaya2009repeated}. What are the consequences of framing the off-switch problem in such a dynamic setting is yet another question that we are planing to tackle next.

\newpage
\appendix

\section{Useful results}
In this section, we have listed some useful results involving Gaussian integrals \cite{owen1980table} that we will use in the proofs.

\begin{lemma}
List of useful Gaussian integrals:
\begin{align}
\label{eq:ResGPDFint}
 \int _{-\infty }^{\infty }\phi (x)\phi (a+bx)\,dx&={\frac {1}{\sqrt {1+b^{2}}}}\phi \left({\frac {a}{\sqrt {1+b^{2}}}}\right),\\
\label{eq:Resprobitint}
 \int _{-\infty }^{\infty }\Phi (a+bx)\phi (x)\,dx&=\Phi \left({\frac {a}{\sqrt {1+b^{2}}}}\right),\\
 \label{eq:xResprobitint}
 \int _{-\infty }^{\infty }x\Phi (a+bx)\phi (x)\,dx&={\frac {b}{\sqrt {1+b^{2}}}}\phi \left({\frac {a}{\sqrt {1+b^{2}}}}\right),
\end{align}
where $\Phi,\phi$ are the CDF and, respectively, PDF of a standard normal distribution.
\end{lemma}
The \textit{inverse Mills ratio} states \cite{grimmett2001probability}:
\begin{lemma}
For $x \sim N(m,s^2)$, it states that:
\begin{equation}
\label{eq:mills}
\begin{aligned}
&E[xI_{\{a\leq x \leq b\}}]\\
&=m \left(\Phi\left(\tfrac{b-m}{s}\right)-\Phi\left(\tfrac{a-m}{s}\right)\right)-s \left(\phi\left(\tfrac{b-m}{s}\right)-\phi\left(\tfrac{a-m}{s}\right)\right).
\end{aligned}
\end{equation}
\end{lemma}
From the above lemma, we can prove:
\begin{lemma}
For $x \sim N(m,s^2)$, 
\begin{equation}
\label{eq:absx}
\begin{aligned}
E[|x|]&=m \left(1-2\Phi\left(\tfrac{-m}{s}\right)\right)+2s \phi\left(\tfrac{-m}{s}\right).\\
\end{aligned}
\end{equation}
\end{lemma}
\begin{proof}
Rewrite $|x|=xI_{\{x \geq 0\}}-xI_{\{x<0\}}$ and apply \eqref{eq:mills}:
\begin{equation}
\label{eq:millsabs0}
\begin{aligned}
&E[xI_{\{x \geq 0\}}]=m \left(1-\Phi\left(\tfrac{-m}{s}\right)\right)-s \left(-\phi\left(\tfrac{-m}{s}\right)\right)
\end{aligned}
\end{equation}
and
\begin{equation}
\label{eq:millsabs1}
\begin{aligned}
&E[xI_{\{x< 0\}}]=m \left(\Phi\left(\tfrac{-m}{s}\right)\right)-s \left(\phi\left(\tfrac{-m}{s}\right)\right)
\end{aligned}
\end{equation}
and, therefore,
\begin{equation}
\label{eq:millsabs2}
\begin{aligned}
E[|x|]&=m \left(1-2\Phi\left(\tfrac{-m}{s}\right)\right)+2s \phi\left(\tfrac{-m}{s}\right).
\end{aligned}
\end{equation}

\end{proof}
Finally, we prove the following main lemma which we will use to prove the results in the paper.
\begin{lemma}
\label{lem:preference_learn}
Consider $x,o \in \mathcal{X}$ and assume that 
\begin{equation}
\label{eq:multiprior3proof}
\begin{bmatrix}
\nu(x) \\ 
\nu(o) 
\end{bmatrix} \sim N\left(\begin{bmatrix}
\mu_p(x) \\ 
\mu_p(o) 
\end{bmatrix},\begin{bmatrix}
K_p(x,x) & K_p(x,o) \\ 
K_p(o,x) & K_p(o,o) 
\end{bmatrix}\right),
\end{equation}
and $n(x),n(o) \sim N(0,\sigma^2)$ (independent noise). Then we have that
\begin{equation}
\label{eq:res}
\begin{aligned}
&E[\nu(x)I_{\{\nu(x)+n(x)>\nu(o)+n(o)\}}]\\
&=\mu_p(x)\left(1-\Phi\left(\tfrac{(\mu_p(o)-\mu_p(x))}{\sqrt{K_p(x,x)+2\sigma^2+K_p(o,o)-2K_p(x,o)}}\right)\right)\\
& + \tfrac{K_p(x,x)-K_p(x,o)}{\sqrt {K_p(x,x)+2\sigma^2+K_p(o,o)-2K_p(x,o)}}\\
&~~\cdot \phi\left(\tfrac{\mu_p(o)-\mu_p(x)}{\sqrt{K_p(x,x)+2\sigma^2+K_p(o,o)-2K_p(x,o)}}\right)
\end{aligned}
\end{equation}

\end{lemma}
\begin{proof}
We will compute $E[\nu(x)I_{\{\nu(x)>\nu(o)+n(o)-n(x)\}}]$ in two steps. First, we assume that $\nu(o),n(o)-n(x)$ are given and, therefore, we condition the joint PDF of $\nu(x),n(x),\nu(o),n(o)$ on $\nu(o),n(o)$. Since only the variables $\nu(x),\nu(o)$ are dependent, then we have
\begin{equation}
\begin{aligned}
&p(\nu(x)|\nu(o))=\\
&{\scriptstyle N\left(\nu(x);\mu_p(x)+\tfrac{K_p(x,o)}{K_p(o,o)}(\nu(o)-\mu_p(o)),
K_p(x,x)-\tfrac{K_p^2(x,o)}{K_p(o,o)}\right)}.
\end{aligned}
\end{equation}

Therefore, we can apply \eqref{eq:mills} conditionally on  $\nu(o),n(o)-n(x)$ which leads to
\begin{equation}
\label{eq:millscond}
\begin{aligned}
&E[\nu(x)I_{\{\nu(x)>\nu(o)+n(o)-n(x)\}}|\nu(o),n(o),n(x)]\\
&=m_1 \left(1-\Phi\left(\tfrac{\nu(o)+n(o)-n(x)-m_1}{\sigma_1}\right)\right)+\sigma_1 \phi\left(\tfrac{\nu(o)+n(o)-n(x)-m_1}{\sigma_1}\right)
\end{aligned}
\end{equation}
Now observe that
\begin{equation}
\label{eq:intm1}
\begin{aligned}
E[m_1]&=\int \left(\mu_p(x)+\tfrac{K_p(x,o)}{K_p(o,o)}(\nu(o)-\mu_p(o))\right)\\
&N(\nu(o);\mu_p(o),K_p(o,o))d\nu(o)
d\nu(o)=\mu_p(x).
\end{aligned}
\end{equation}
and
\begin{equation}
\label{eq:intm1}
\begin{aligned}
&E\left[m_1\Phi\left(\tfrac{\nu(o)+n(o)-n(x)-m_1}{\sigma_1}\right)\right]\\
&=E\left[\left(\mu_p(x)+\tfrac{K_p(x,o)}{K_p(o,o)}(\nu(o)-\mu_p(o))\right)\Phi\left(\tfrac{\nu(o)+n(o)-n(x)-m_1}{\sigma_1}\right)\right]\\
&=\left(\mu_p(x)-\tfrac{K_p(x,o)}{K_p(o,o)}\mu_p(o)\right) E\left[\Phi\left(\tfrac{\nu(o)+n(o)-n(x)-m_1}{\sigma_1}\right)\right]\\
&+\tfrac{K_p(x,o)}{K_p(o,o)}E\left[ \nu(o)\Phi\left(\tfrac{\nu(o)+n(o)-n(x)-m_1}{\sigma_1}\right)\right]
\end{aligned}
\end{equation}
The expectations are with respect to $\nu(o),n(o),n(x)$.
Now we use \eqref{eq:Resprobitint} to get the following result:
\begin{equation}
\label{eq:probitint1}
\begin{aligned}
&\int \Phi\left(\tfrac{\nu(o)+n(o)-n(x)-m_1}{\sigma_1}\right)N(n(o)-n(x);0,2\sigma^2)dn(o)\\
&=\Phi\left(\tfrac{\nu(o)-m_1}{\sqrt{\sigma_1^2+2\sigma^2}}\right),
\end{aligned}
\end{equation}
and so:
\begin{equation}
\label{eq:probitint2}
\begin{aligned}
&E\left[\Phi\left(\tfrac{\nu(o)+n(o)-n(x)-m_1}{\sigma_1}\right)\right]\\
&=\int \Phi\left(\tfrac{\nu(o)-m_1}{\sqrt{\sigma_1^2+2\sigma^2}}\right)N(\nu(o);\mu_p(o),K_p(o,o))d\nu(o)\\
&=\int \Phi\left(\tfrac{\nu(o)\tfrac{K_p(o,o)-K_p(x,o)}{K_p(o,o)}+m_2}{\sqrt{\sigma_1^2+2\sigma^2}}\right)\\
&~~~~~~~~N(\nu(o);\mu_p(o),K_p(o,o))d\nu(o)\\
&=\int \Phi\left(\tfrac{z\tfrac{K_p(o,o)-K_p(x,o)}{\sqrt{K_p(o,o)}}+m_2+\tfrac{K_p(o,o)-K_p(x,o)}{K_p(o,o)}\mu_p(o)}{\sqrt{\sigma_1^2+2\sigma^2}}\right)\\
&~~~~~~~~N(z;0,1)dz\\
&=\Phi\left(\tfrac{\mu_p(o)\tfrac{K_p(o,o)-K_p(x,o)}{\sqrt{K_p(o,o)}}+m_2\sqrt{K_p(o,o)}}{\sqrt{K_p(o,o)(\sigma_1^2+2\sigma^2)+(K_p(o,o)-K_p(x,o))^2}}\right)\\
&=\Phi\left(\tfrac{\sqrt{K_p(o,o)}(\mu_p(o)-\mu_p(x))}{\sqrt{K_p(o,o)(\sigma_1^2+2\sigma^2)+(K_p(o,o)-K_p(x,o))^2}}\right),\\
\end{aligned}
\end{equation}
with $m_2=\tfrac{-K_p(o,o)\mu_p(x)+K_p(x,o)\mu_p(o)}{
K_p(o,o)}$. Similarly, we have that
\begin{equation}
\label{eq:probitint2bis}
\begin{aligned}
&E\left[\nu(o)\Phi\left(\tfrac{\nu(o)+n(o)-n(x)-m_1}{\sigma_1}\right)\right]\\
&=\int \nu(o) \Phi\left(\tfrac{\nu(o)-m_1}{\sqrt{\sigma_1^2+2\sigma^2}}\right)N(\nu(o);\mu_p(o),K_p(o,o))d\nu(o)\\
&=\int \nu(o) \Phi\left(\tfrac{\nu(o)\tfrac{K_p(o,o)-K_p(x,o)}{K_p(o,o)}+m_2}{\sqrt{\sigma_1^2+2\sigma^2}}\right)\\
&~~~~~~~~N(\nu(o);\mu_p(o),K_p(o,o))d\nu(o)\\
&=\int \Phi\left(\tfrac{z\tfrac{K_p(o,o)-K_p(x,o)}{\sqrt{K_p(o,o)}}+m_2+\tfrac{K_p(o,o)-K_p(x,o)}{K_p(o,o)}\mu_p(o)}{\sqrt{\sigma_1^2+2\sigma^2}}\right)\\
& ~~~~\left(z \sqrt{K_p(o,o)}+\mu_p(o)\right)N(z;0,1)dz\\
\end{aligned}
\end{equation}
We separate the sum:
\begin{equation}
\label{eq:probitint2bis1}
\begin{aligned}
&\int \Phi\left(\tfrac{z\tfrac{K_p(o,o)-K_p(x,o)}{\sqrt{K_p(o,o)}}+m_2+\tfrac{K_p(o,o)-K_p(x,o)}{K_p(o,o)}\mu_p(o)}{\sqrt{\sigma_1^2+2\sigma^2}}\right)\\
& ~~~~\mu_p(o) N(z;0,1)dz\\
&=\mu_p(o)\Phi\left(\tfrac{\mu_p(o)\tfrac{K_p(o,o)-K_p(x,o)}{\sqrt{K_p(o,o)}}+m_2\sqrt{K_p(o,o)}}{\sqrt{K_p(o,o)(\sigma_1^2+2\sigma^2)+(K_p(o,o)-K_p(x,o))^2}}\right)\\
&=\mu_p(o)\Phi\left(\tfrac{\sqrt{K_p(o,o)}(\mu_p(o)-\mu_p(x))}{\sqrt{K_p(o,o)(\sigma_1^2+2\sigma^2)+(K_p(o,o)-K_p(x,o))^2}}\right).\\
\end{aligned}
\end{equation}
The other term in the sum
\begin{equation}
\label{eq:probitint2bis2}
\begin{aligned}
&\int \Phi\left(\tfrac{z\tfrac{K_p(o,o)-K_p(x,o)}{\sqrt{K_p(o,o)}}+m_2+\tfrac{K_p(o,o)-K_p(x,o)}{K_p(o,o)}\mu_p(o)}{\sqrt{\sigma_1^2+2\sigma^2}}\right)\\
& ~~~~z \sqrt{K_p(o,o)} N(z;0,1)dz\\
&=\tfrac{\sqrt{K_p(o,o)}(K_p(o,o)-K_p(x,o))}{\sqrt{K_p(o,o)(\sigma_1^2+\sigma^2)+(K_p(o,o)-K_p(x,o))^2}}\\
&\phi\left(\tfrac{\mu_p(o)\tfrac{K_p(o,o)-K_p(x,o)}{\sqrt{K_p(o,o)}}+m_2\sqrt{K_p(o,o)}}{\sqrt{K_p(o,o)(\sigma_1^2+2\sigma^2)+(K_p(o,o)-K_p(x,o))^2}}\right)\\
&=\tfrac{\sqrt{K_p(o,o)}(K_p(o,o)-K_p(x,o))}{\sqrt{K_p(o,o)(\sigma_1^2+2\sigma^2)+(K_p(o,o)-K_p(x,o))^2}}\\
&\phi\left(\tfrac{\sqrt{K_p(o,o)}(\mu_p(o)-\mu_p(x))}{\sqrt{K_p(o,o)(\sigma_1^2+2\sigma^2)+(K_p(o,o)-K_p(x,o))^2}}\right).\\
\end{aligned}
\end{equation}
where we have used \eqref{eq:xResprobitint}. 
Finally, we consider
\begin{equation}
\label{eq:probitint3pre}
\begin{aligned}
&\int \phi\left(\tfrac{\nu(o)+n(o)-n(x)-m_1}{\sigma_1}\right)N(n(o)-n(x);0,2\sigma^2)dn(o)\\
&=\tfrac{\sigma_1}{\sqrt{\sigma_1^2+2\sigma^2}}\phi\left(\tfrac{\nu(o)-m_1}{\sqrt{\sigma_1^2+2\sigma^2}}\right),\\
\end{aligned}
\end{equation}
where the last equality follows by \eqref{eq:ResGPDFint}. We use \eqref{eq:probitint2} to get:
\begin{equation}
\label{eq:probitint3} 
\begin{aligned}
&\int \tfrac{\sigma_1}{\sqrt{\sigma_1^2+2\sigma^2}}\phi\left(\tfrac{\nu(o)-m_1}{\sqrt{\sigma_1^2+2\sigma^2}}\right)N(\nu(o);\mu_p(o),K_p(o,o))d\nu(o)\\
&=\int \tfrac{\sigma_1}{\sqrt{\sigma_1^2+2\sigma^2}}\phi\left(\tfrac{\nu(o)\tfrac{K_p(o,o)-K_p(x,o)}{K_p(o,o)}+m_2}{\sqrt{\sigma_1^2+2\sigma^2}}\right)\\
&N(\nu(o);\mu_p(o),K_p(o,o))d\nu(o)\\
&=\int \tfrac{\sigma_1}{\sqrt{\sigma_1^2+2\sigma^2}}\phi\left(\tfrac{z\tfrac{K_p(o,o)-K_p(x,o)}{\sqrt{K_p(o,o)}}+m_2+\tfrac{K_p(o,o)-K_p(x,o)}{K_p(o,o)}\mu_p(o)}{\sqrt{\sigma_1^2+2\sigma^2}}\right)\\
&N(z;0,1)d\nu(o)\\
&=\tfrac{\sqrt{K_p(o,o)}\sigma_1}{\sqrt{K_p(o,o)(\sigma_1^2+2\sigma^2)+(K_p(o,o)-K_p(x,o))^2}}\\
&\phi\left(\tfrac{\sqrt{K_p(o,o)}(\mu_p(o)-\mu_p(x))}{\sqrt{K_p(o,o)(\sigma_1^2+2\sigma^2)+(K_p(o,o)-K_p(x,o))^2}}\right).\\
\end{aligned}
\end{equation}

Therefore, from \eqref{eq:millscond} and \eqref{eq:probitint2}--\eqref{eq:probitint3}, we obtain
\begin{equation}
\label{eq:millsthird}
\begin{aligned}
&E[(\nu(x)+n(x))I_{\{\nu(x)+n(x)>\nu(o)+n(o)\}}]=\mu_p(x)\\
&-\left(\mu_p(x)-\tfrac{K_p(x,o)}{K_p(o,o)}\mu_p(o)\right)\\
&\cdot \Phi\left(\tfrac{\sqrt{K_p(o,o)}(\mu_p(o)-\mu_p(x))}{\sqrt{K_p(o,o)(\sigma_1^2+2\sigma^2)+(K_p(o,o)-K_p(x,o))^2}}\right)\\
& - \tfrac{K_p(x,o)}{K_p(o,o)} \mu_p(o)\Phi\left(\tfrac{\sqrt{K_p(o,o)}(\mu_p(o)-\mu_p(x))}{\sqrt{K_p(o,o)(\sigma_1^2+2\sigma^2)+(K_p(o,o)-K_p(x,o))^2}}\right)\\
& - \tfrac{K_p(x,o)}{K_p(o,o)} \tfrac{\sqrt{K_p(o,o)}(K_p(o,o)-K_p(x,o))}{\sqrt{K_p(o,o)(\sigma_1^2+2\sigma^2)+(K_p(o,o)-K_p(x,o))^2}}\\
&\phi\left(\tfrac{\sqrt{K_p(o,o)}(\mu_p(o)-\mu_p(x))}{\sqrt{K_p(o,o)(\sigma_1^2+2\sigma^2)+(K_p(o,o)-K_p(x,o))^2}}\right)\\
&+\tfrac{\sqrt{K_p(o,o)}\sigma^2_1}{\sqrt{K_p(o,o)(\sigma_1^2+2\sigma^2)+(K_p(o,o)-K_p(x,o))^2}}\\
&\phi\left(\tfrac{\sqrt{K_p(o,o)}(\mu_p(o)-\mu_p(x))}{\sqrt{K_p(o,o)(\sigma_1^2+2\sigma^2)+(K_p(o,o)-K_p(x,o))^2}}\right)\\
&=\mu_p(x)\left(1-\Phi\left(\tfrac{\sqrt{K_p(o,o)}(\mu_p(o)-\mu_p(x))}{\sqrt{K_p(o,o)(\sigma_1^2+2\sigma^2)+(K_p(o,o)-K_p(x,o))^2}}\right)\right)\\
& + \tfrac{\sqrt{K_p(o,o)}(K_p(x,x)-K_p(x,o))}{\sqrt{K_p(o,o)(\sigma_1^2+2\sigma^2)+(K_p(o,o)-K_p(x,o))^2}}\\
&\phi\left(\tfrac{\sqrt{K_p(o,o)}(\mu_p(o)-\mu_p(x))}{\sqrt{K_p(o,o)(\sigma_1^2+2\sigma^2)+(K_p(o,o)-K_p(x,o))^2}}\right)\\
\end{aligned}
\end{equation}

Note that
\begin{equation}
\label{eq:var00}
\begin{aligned}
&K_p(o,o)(\sigma_1^2+2\sigma^2)+(K_p(o,o)-K_p(x,o))^2\\
&=K_p(o,o)\left(K_p(x,x)-\tfrac{K_p^2(x,o)}{K_p(o,o)}+2\sigma^2\right)\\
&+(K_p(o,o)-K_p(x,o))^2\\
&=K_p(o,o)K_p(x,x)-K_p^2(x,o)+2\sigma^2K_p(o,o)\\
&+K_p^2(o,o)+K_p^2(x,o)-2K_p(x,o)K_p(o,o)\\
&=K_p(o,o)(K_p(x,x)+2\sigma^2+K_p(o,o)-2K_p(x,o)).\\
\end{aligned}
\end{equation}
Therefore, we have that
\begin{equation}
\label{eq:millsthird11}
\begin{aligned}
&E[\nu(x)I_{\{\nu(x)+n(x)>\nu(o)+n(o)\}}]=\\
&=\mu_p(x)\left(1-\Phi\left(\tfrac{(\mu_p(o)-\mu_p(x))}{\sqrt{K_p(x,x)+2\sigma^2+K_p(o,o)-2K_p(x,o)}}\right)\right)\\
& + \tfrac{K_p(x,x)-K_p(x,o)}{\sqrt {K_p(x,x)+2\sigma^2+K_p(o,o)-2K_p(x,o)}}\\
&\phi\left(\tfrac{\mu_p(o)-\mu_p(x)}{\sqrt{K_p(x,x)+2\sigma^2+K_p(o,o)-2K_p(x,o)}}\right)
\end{aligned}
\end{equation}
    \end{proof}

\section{Proofs}
\label{app:proof}
We now move on to the main results.

\paragraph{Proof of Lemma \ref{lem:1}}
The expected value for $DEF$  follows from Lemma \ref{lem:preference_learn} by summing $E[\nu(x)I_{\{\nu(x)+n(x)>\nu(o)+n(o)\}}]$ and $E[\nu(o)I_{\{\nu(x)+n(x)<\nu(o)+n(o)\}}]$. The expected values for $IMM,OFF$ are straightforward.

\paragraph{Proof of Proposition \ref{prop:1}}
The results follow from Lemma \ref{lem:1}  by considering whether or not the limits $K_0(x,x),K_0(o,o),K_0(o,x) \rightarrow 0$ and $\sigma \rightarrow 0$ are taken. We make the assumption that whenever $K_o(x,x),K_o(o,o),K_o(o,x) \rightarrow 0$  it implies that $K_p(x,x),K_p(o,o),K_p(o,x) \rightarrow 0$ (a-priori we have a Dirac's delta). Since $K_p$ depends on both $K_0$ and $\sigma$, we always take the limit with respect to $\sigma$ first.

 If $S$ is \textbf{rational} and $R$ has \textbf{no uncertainty}, then the expected payoffs can be computed from \eqref{eq:DEFexp}, \eqref{eq:IMMexp} and \eqref{eq:OFFexp}.
  The values are
 $E[DEF]=\max(\mu_p(x),\mu_p(o))$
  $E[IMM]=\mu_p(x)$ and $E[OFF]=\mu_p(o)$. Therefore, DEF is never dominated.
  
   If $S$ is \textbf{bounded-rational} and $R$ has \textbf{no uncertainty}, then $\sigma>0$ and the payoffs are:
   $E[DEF]=p\mu_p(x)+(1-p)\mu_p(o)$, 
  $E[IMM]=\mu_p(x)$ and $E[OFF]=\mu_p(o)$, where
  $p=\Phi\left(\tfrac{\mu_p(x)-\mu_p(o)}{\sqrt{2\sigma^2}}\right)$
   Therefore, $p \in (0,1)$ and   
   DEF is never optimal.

If $S$ is \textbf{rational} and $R$ has \textbf{uncertainty}, then 
$E[DEF]=p\mu_p(x)+(1-p)\mu_p(o)+e$,   $E[IMM]=\mu_p(x)$ and $E[OFF]=\mu_p(o)$, where  $p=\Phi\left(\tfrac{\mu_p(x)-\mu_p(o)}{\sqrt{K_p(o,o)+K_p(x,x)-2K_p(x,o)}}\right)$ and 
$$
 \resizebox{0.91\hsize}{!}{$
\begin{aligned}
e&= \tfrac{K_p(x,x)-K_p(x,o)}{\sqrt {K_p(x,x)+K_p(o,o)-2K_p(x,o)}}
\phi\left(\tfrac{\mu_p(o)-\mu_p(x)}{\sqrt{K_p(x,x)+K_p(o,o)-2K_p(x,o)}}\right)\\
& + \tfrac{K_p(o,o)-K_p(x,o)}{\sqrt {K_p(x,x)+K_p(o,o)-2K_p(x,o)}}\phi\left(\tfrac{\mu_p(x)-\mu_p(o)}{\sqrt{K_p(x,x)+K_p(o,o)-2K_p(x,o)}}\right).
\end{aligned}$}
$$
When $\bH$ is rational, then 
\begin{equation}
\label{eq:payoffprop1Rproof}
\begin{aligned}
&u_R(t,DEF,b^*(DEF))=\nu(o)I_{\{\nu(o)>\nu(x)\}}
\\&+\nu(x)I_{\{\nu(x)>\nu(o)\}}=\max(\nu(o),\nu(x))
\end{aligned}
\end{equation}
$\max$ is a convex function and, therefore, by Jensen's inequality 
$E[\max(\nu(o),\nu(x))]\geq \max(E[\nu(o)],E[\nu(x)])$. Therefore, $p\mu_p(x)+(1-p)\mu_p(o)+e\geq \max(\mu_p(x),\mu_p(o))$ and DEF is always optimal.

The last case follows directly from Lemma \ref{lem:1}.

\paragraph{Proof of Corollary \ref{co:2}}
The results follow from Proposition \ref{prop:1} (and \eqref{eq:absx}) after subtracting $\beta$ to the expected payoff for DEF, where
\begin{align}
\nonumber
\beta &=\gamma E[|\nu(o)|] =\gamma \mu_p(o) \left(1-2\Phi\left(\tfrac{-\mu_p(o)}{\sqrt{K_p(o,o)}}\right)\right)\\
&+2\gamma\sqrt{K_p(o,o)} \phi\left(\tfrac{-\mu_p(o)}{\sqrt{K_p(o,o)}}\right).
\end{align}
 If $S$ is \textbf{rational} and $R$ has \textbf{no uncertainty}, then the expected payoffs can be computed from \eqref{eq:DEFexp}, \eqref{eq:IMMexp} and \eqref{eq:OFFexp}.
  The values are
 $E[DEF]=\max(\mu_p(x),\mu_p(o))-\gamma'|\mu_p(o)|$
  $E[IMM]=\mu_p(x)$ and $E[OFF]=\mu_p(o)$. Therefore, DEF is never optimal.

   If $S$ is \textbf{bounded-rational} and $R$ has \textbf{no uncertainty}, then $\sigma>0$ and the payoffs are:
   $E[DEF]=p\mu_p(x)+(1-p)\mu_p(o)-\gamma'|\mu_p(o)|$
  $E[IMM]=\mu_p(x)$ and $E[OFF]=\mu_p(o)$, where
  $p=\Phi\left(\tfrac{\mu_p(x)-\mu_p(o)}{\sqrt{2\sigma^2}}\right)$
   Therefore, $p \in (0,1)$ and   
   DEF is never optimal.

   If $S$ is \textbf{rational} and $R$ has \textbf{uncertainty}, then 
$E[DEF]=p\mu_p(x)+(1-p)\mu_p(o)+e-\gamma'|\mu_p(o)|$
  $E[IMM]=\mu_p(x)$ and $E[OFF]=\mu_p(o)$. Therefore, 
  DEF is optimal if 
$p\mu_p(x) + (1-p)\mu_p(o)+e-\gamma'|\mu_p(o)|\geq \max\left(\mu_p(x) ,\mu_p(o)\right)$.
The last case follows similarly from Proposition \ref{prop:1}

\paragraph{Proof of Lemma \ref{lem:2}}
The expected value for $DEF$  is equal to the sum of $E[\nu(x)I_{\{\nu(x)>\nu(o)+\sigma)\}}]$, $E[\nu(o)I_{\{\nu(o)>\nu(x)+\sigma)\}}]$ and $\{E[\nu(x)I_{\{|\nu(x)-\nu(o)|\leq \sigma)\}}],E[\nu(o)I_{\{|\nu(x)-\nu(o)|\leq \sigma)\}}]\}$.
 For fixed $\nu(x)$, we have that
 \begin{equation}
\begin{aligned}
&p(\nu(x)|\nu(o))=N(\nu(x);m_1,\sigma_1^2)=\\
&{\scriptstyle N\left(\nu(x);\mu_p(x)+\tfrac{K_p(x,o)}{K_p(o,o)}(\nu(o)-\mu_p(o)),
K_p(x,x)-\tfrac{K_p^2(x,o)}{K_p(o,o)}\right)}.
\end{aligned}
\end{equation}
Therefore, we can apply \eqref{eq:mills} conditionally on  $\nu(o)$ which leads to
\begin{equation}
\label{eq:millscond42}
\begin{aligned}
&E[\nu(x)I_{\{\nu(x)>\nu(o)+\sigma\}}|\nu(o)]\\
&=m_1 \left(1-\Phi\left(\tfrac{\nu(o)+\sigma-m_1}{\sigma_1}\right)\right)+\sigma_1 \phi\left(\tfrac{\nu(o)+\sigma-m_1}{\sigma_1}\right)
\end{aligned}
\end{equation}
Now observe that
\begin{equation}
\label{eq:intm142}
\begin{aligned}
E[m_1]&=\int \left(\mu_p(x)+\tfrac{K_p(x,o)}{K_p(o,o)}(\nu(o)-\mu_p(o))\right)\\
&N(\nu(o);\mu_p(o),K_p(o,o))d\nu(o)
d\nu(o)=\mu_p(x),
\end{aligned}
\end{equation}
and
\begin{equation}
\label{eq:intm1421}
\begin{aligned}
&E\left[m_1\Phi\left(\tfrac{\nu(o)+\sigma-m_1}{\sigma_1}\right)\right]\\
&=E\left[\left(\mu_p(x)+\tfrac{K_p(x,o)}{K_p(o,o)}(\nu(o)-\mu_p(o))\right)\Phi\left(\tfrac{\nu(o)+\sigma-m_1}{\sigma_1}\right)\right]\\
&=\left(\mu_p(x)-\tfrac{K_p(x,o)}{K_p(o,o)}\mu_p(o)\right) E\left[\Phi\left(\tfrac{\nu(o)+\sigma-m_1}{\sigma_1}\right)\right]\\
&+\tfrac{K_p(x,o)}{K_p(o,o)}E\left[ \nu(o)\Phi\left(\tfrac{\nu(o)+\sigma-m_1}{\sigma_1}\right)\right]
\end{aligned}
\end{equation}
The expectations are with respect to $\nu(o)$.
Now we use \eqref{eq:Resprobitint} to get 

\begin{equation}
\label{eq:probitint242}
\begin{aligned}
&E\left[\Phi\left(\tfrac{\nu(o)+\sigma-m_1}{\sigma_1}\right)\right]\\
&=\int \Phi\left(\tfrac{\nu(o)+\sigma-m_1}{\sigma_1}\right)N(\nu(o);\mu_p(o),K_p(o,o))d\nu(o)\\
&=\int \Phi\left(\tfrac{\nu(o)\tfrac{K_p(o,o)-K_p(x,o)}{K_p(o,o)}+m_2}{\sigma_1}\right)\\
&~~~~~~~~N(\nu(o);\mu_p(o),K_p(o,o))d\nu(o)\\
&=\int \Phi\left(\tfrac{z\tfrac{K_p(o,o)-K_p(x,o)}{\sqrt{K_p(o,o)}}+m_2+\tfrac{K_p(o,o)-K_p(x,o)}{K_p(o,o)}\mu_p(o)}{\sigma_1}\right)\\
&~~~~~~~~N(z;0,1)dz\\
&=\Phi\left(\tfrac{\mu_p(o)\tfrac{K_p(o,o)-K_p(x,o)}{\sqrt{K_p(o,o)}}+m_2\sqrt{K_p(o,o)}}{\sqrt{K_p(o,o)\sigma_1^2+(K_p(o,o)-K_p(x,o))^2}}\right)\\
&=\Phi\left(\tfrac{\sqrt{K_p(o,o)}(\mu_p(o)+\sigma-\mu_p(x))}{\sqrt{K_p(o,o)\sigma_1^2+(K_p(o,o)-K_p(x,o))^2}}\right),\\
\end{aligned}
\end{equation}
with $m_2=\tfrac{K_p(o,o)(\sigma-\mu_p(x))+K_p(x,o)\mu_p(o)}{
K_p(o,o)}$. Similarly, we have that
\begin{equation}
\label{eq:probitint2bis42}
\begin{aligned}
&E\left[\nu(o)\Phi\left(\tfrac{\nu(o)+\sigma-m_1}{\sigma_1}\right)\right]\\
&=\int \nu(o) \Phi\left(\tfrac{\nu(o)+\sigma-m_1}{\sigma_1}\right) N(\nu(o);\mu_p(o),K_p(o,o))d\nu(o)\\
&=\int \nu(o) \Phi\left(\tfrac{\nu(o)\tfrac{K_p(o,o)-K_p(x,o)}{K_p(o,o)}+m_2}{\sigma_1}\right)\\
&~~~~~~~~N(\nu(o);\mu_p(o),K_p(o,o))d\nu(o)\\
&=\int \Phi\left(\tfrac{z\tfrac{K_p(o,o)-K_p(x,o)}{\sqrt{K_p(o,o)}}+m_2+\tfrac{K_p(o,o)-K_p(x,o)}{K_p(o,o)}\mu_p(o)}{\sigma_1}\right)\\
& ~~~~\left(z \sqrt{K_p(o,o)}+\mu_p(o)\right)N(z;0,1)dz\\
\end{aligned}
\end{equation}
We separate the sum:
\begin{equation}
\label{eq:probitint2bis142}
\begin{aligned}
&\int \Phi\left(\tfrac{z\tfrac{K_p(o,o)-K_p(x,o)}{\sqrt{K_p(o,o)}}+m_2+\tfrac{K_p(o,o)-K_p(x,o)}{K_p(o,o)}\mu_p(o)}{\sigma_1}\right)\\
& ~~~~\mu_p(o) N(z;0,1)dz\\
&=\mu_p(o)\Phi\left(\tfrac{\mu_p(o)\tfrac{K_p(o,o)-K_p(x,o)}{\sqrt{K_p(o,o)}}+m_2\sqrt{K_p(o,o)}}{\sqrt{K_p(o,o)\sigma_1^2+(K_p(o,o)-K_p(x,o))^2}}\right)\\
&=\mu_p(o)\Phi\left(\tfrac{\sqrt{K_p(o,o)}(\mu_p(o)+\sigma-\mu_p(x))}{\sqrt{K_p(o,o)\sigma_1^2+(K_p(o,o)-K_p(x,o))^2}}\right).\\
\end{aligned}
\end{equation}
The other term in the sum
\begin{equation}
\label{eq:probitint2bis242}
\begin{aligned}
&\int \Phi\left(\tfrac{z\tfrac{K_p(o,o)-K_p(x,o)}{\sqrt{K_p(o,o)}}+m_2+\tfrac{K_p(o,o)-K_p(x,o)}{K_p(o,o)}\mu_p(o)}{\sigma_1}\right)\\
& ~~~~z \sqrt{K_p(o,o)} N(z;0,1)dz\\
&=\tfrac{\sqrt{K_p(o,o)}(K_p(o,o)-K_p(x,o))}{\sqrt{K_p(o,o)\sigma_1^2+(K_p(o,o)-K_p(x,o))^2}}\\
&\phi\left(\tfrac{\mu_p(o)\tfrac{K_p(o,o)-K_p(x,o)}{\sqrt{K_p(o,o)}}+m_2\sqrt{K_p(o,o)}}{\sqrt{K_p(o,o)(\sigma_1^2+2\sigma^2)+(K_p(o,o)-K_p(x,o))^2}}\right)\\
&=\tfrac{\sqrt{K_p(o,o)}(K_p(o,o)-K_p(x,o))}{\sqrt{K_p(o,o)\sigma_1^2+(K_p(o,o)-K_p(x,o))^2}}\\
&\phi\left(\tfrac{\sqrt{K_p(o,o)}(\mu_p(o)+\sigma-\mu_p(x))}{\sqrt{K_p(o,o)\sigma_1^2+(K_p(o,o)-K_p(x,o))^2}}\right).\\
\end{aligned}
\end{equation}
where we have used \eqref{eq:xResprobitint}. 
Finally, we consider
\begin{equation}
\label{eq:probitint342} 
\begin{aligned}
&\int \phi\left(\tfrac{\nu(o)+\sigma-m_1}{\sigma_1}\right)N(\nu(o);\mu_p(o),K_p(o,o))d\nu(o)\\
&=\int \phi\left(\tfrac{z\tfrac{K_p(o,o)-K_p(x,o)}{\sqrt{K_p(o,o)}}+m_2+\tfrac{K_p(o,o)-K_p(x,o)}{K_p(o,o)}\mu_p(o)}{\sigma_1}\right)\\
&N(z;0,1)d\nu(o)\\
&=\tfrac{\sqrt{K_p(o,o)}\sigma_1}{\sqrt{K_p(o,o)_p\sigma_1^2+(K_p(o,o)-K_p(x,o))^2}}\\
&\phi\left(\tfrac{\sqrt{K_p(o,o)}(\mu_p(o)+\sigma-\mu_p(x))}{\sqrt{K_p(o,o)\sigma_1^2+(K_p(o,o)-K_p(x,o))^2}}\right).\\
\end{aligned}
\end{equation}

Therefore, from \eqref{eq:millscond} and \eqref{eq:probitint2}--\eqref{eq:probitint3}, we obtain
\begin{equation}
\label{eq:millsthird42}
\begin{aligned}
&E[\nu(x)I_{\{\nu(x)>\nu(o)+\sigma\}}]=\mu_p(x)\\
&-\left(\mu_p(x)-\tfrac{K_p(x,o)}{K_p(o,o)}\mu_p(o)\right)\\
&\cdot \Phi\left(\tfrac{\sqrt{K_p(o,o)}(\mu_p(o)+\sigma-\mu_p(x))}{\sqrt{K_p(o,o)\sigma_1^2+(K_p(o,o)-K_p(x,o))^2}}\right)\\
& - \tfrac{K_p(x,o)}{K_p(o,o)} \mu_p(o)\Phi\left(\tfrac{\sqrt{K_p(o,o)}(\mu_p(o)+\sigma-\mu_p(x))}{\sqrt{K_p(o,o)\sigma_1^2+(K_p(o,o)-K_p(x,o))^2}}\right)\\
& - \tfrac{K_p(x,o)}{K_p(o,o)} \tfrac{\sqrt{K_p(o,o)}(K_p(o,o)-K_p(x,o))}{\sqrt{K_p(o,o)\sigma_1^2+(K_p(o,o)-K_p(x,o))^2}}\\
&\phi\left(\tfrac{\sqrt{K_p(o,o)}(\mu_p(o)+\sigma-\mu_p(x))}{\sqrt{K_p(o,o)\sigma_1^2+(K_p(o,o)-K_p(x,o))^2}}\right)\\
&+\tfrac{\sqrt{K_p(o,o)}\sigma^2_1}{\sqrt{K_p(o,o)\sigma_1^2+(K_p(o,o)-K_p(x,o))^2}}\\
&\phi\left(\tfrac{\sqrt{K_p(o,o)}(\mu_p(o)+\sigma-\mu_p(x))}{\sqrt{K_p(o,o)\sigma_1^2+(K_p(o,o)-K_p(x,o))^2}}\right)\\
&=\mu_p(x)\left(1-\Phi\left(\tfrac{\sqrt{K_p(o,o)}(\mu_p(o)+\sigma-\mu_p(x))}{\sqrt{K_p(o,o)\sigma_1^2+(K_p(o,o)-K_p(x,o))^2}}\right)\right)\\
& + \tfrac{\sqrt{K_p(o,o)}(K_p(x,x)-K_p(x,o))}{\sqrt{K_p(o,o)(\sigma_1^2+2\sigma^2)+(K_p(o,o)-K_p(x,o))^2}}\\
&\phi\left(\tfrac{\sqrt{K_p(o,o)}(\mu_p(o)+\sigma-\mu_p(x))}{\sqrt{K_p(o,o)\sigma_1^2+(K_p(o,o)-K_p(x,o))^2}}\right)\\
\end{aligned}
\end{equation}

Note that
\begin{equation}
\label{eq:var0042}
\begin{aligned}
&K_p(o,o)\sigma_1^2+(K_p(o,o)-K_p(x,o))^2\\
&=K_p(o,o)\left(K_p(x,x)-\tfrac{K_p^2(x,o)}{K_p(o,o)}\right)+(K_p(o,o)-K_p(x,o))^2\\
&=K_p(o,o)K_p(x,x)-K_p^2(x,o)\\
&+K_p^2(o,o)+K_p^2(x,o)-2K_p(x,o)K_p(o,o)\\
&=K_p(o,o)(K_p(x,x)+K_p(o,o)-2K_p(x,o)).\\
\end{aligned}
\end{equation}
Therefore, we have that
\begin{equation}
\label{eq:millsthird1142}
\begin{aligned}
&E[\nu(x)I_{\{\nu(x)>\nu(o)+\sigma\}}]=\\
&=\mu_p(x)\left(1-\Phi\left(\tfrac{(\mu_p(o)+\sigma-\mu_p(x))}{\sqrt{K_p(x,x)+K_p(o,o)-2K_p(x,o)}}\right)\right)\\
& + \tfrac{K_p(x,x)-K_p(x,o)}{\sqrt {K_p(x,x)+K_p(o,o)-2K_p(x,o)}}\\
&\phi\left(\tfrac{\mu_p(o)+\sigma-\mu_p(x)}{\sqrt{K_p(x,x)+K_p(o,o)-2K_p(x,o)}}\right)
\end{aligned}
\end{equation}
and 
\begin{equation}
\label{eq:millsthird1142sec}
\begin{aligned}
&E[\nu(o)I_{\{\nu(o)>\nu(x)+\sigma\}}]=\\
&=\mu_p(o)\left(1-\Phi\left(\tfrac{(\mu_p(x)+\sigma-\mu_p(o))}{\sqrt{K_p(x,x)+K_p(o,o)-2K_p(x,o)}}\right)\right)\\
& + \tfrac{K_p(o,o)-K_p(x,o)}{\sqrt {K_p(x,x)+K_p(o,o)-2K_p(x,o)}}\\
&\phi\left(\tfrac{\mu_p(x)+\sigma-\mu_p(o)}{\sqrt{K_p(x,x)+K_p(o,o)-2K_p(x,o)}}\right)
\end{aligned}
\end{equation}
The other two terms are:
\begin{equation}
\label{eq:millsthird1142sec}
\begin{aligned}
&E[\nu(o)I_{\{|\nu(o)-\nu(x)|\leq \sigma\}}]=\mu_p(o)\\
&-E[\nu(o)I_{\{\nu(o)>\nu(x)+\sigma\}}]-E[\nu(o)I_{\{\nu(x)>\nu(o)+\sigma\}}]\\
&=\mu_p(o)-\mu_p(o)\left(1-\Phi\left(\tfrac{(\mu_p(x)+\sigma-\mu_p(o))}{\sqrt{K_p(x,x)+K_p(o,o)-2K_p(x,o)}}\right)\right)\\
& - \tfrac{K_p(o,o)-K_p(x,o)}{\sqrt {K_p(x,x)+K_p(o,o)-2K_p(x,o)}}\\
&\phi\left(\tfrac{\mu_p(x)+\sigma-\mu_p(o)}{\sqrt{K_p(x,x)+K_p(o,o)-2K_p(x,o)}}\right)\\
&-\mu_p(o)\left(1-\Phi\left(\tfrac{(-\mu_p(x)+\sigma+\mu_p(o))}{\sqrt{K_p(x,x)+K_p(o,o)-2K_p(x,o)}}\right)\right)\\
& + \tfrac{K_p(o,o)-K_p(x,o)}{\sqrt {K_p(x,x)+K_p(o,o)-2K_p(x,o)}}\\
&\phi\left(\tfrac{-\mu_p(x)+\sigma+\mu_p(o)}{\sqrt{K_p(x,x)+K_p(o,o)-2K_p(x,o)}}\right)\\
&=\mu_p(o)-\mu_p(o)\Bigg(2-\Phi\left(\tfrac{(\mu_p(x)+\sigma-\mu_p(o))}{\sqrt{K_p(x,x)+K_p(o,o)-2K_p(x,o)}}\right)\\
&-\Phi\left(\tfrac{(-\mu_p(x)+\sigma+\mu_p(o))}{\sqrt{K_p(x,x)+K_p(o,o)-2K_p(x,o)}}\right)\Bigg)\\
& + \tfrac{K_p(o,o)-K_p(x,o)}{\sqrt {K_p(x,x)+K_p(o,o)-2K_p(x,o)}}\\
&\Bigg(\phi\left(\tfrac{-\mu_p(x)+\sigma+\mu_p(o)}{\sqrt{K_p(x,x)+K_p(o,o)-2K_p(x,o)}}\right)\\
&-\phi\left(\tfrac{\mu_p(x)+\sigma-\mu_p(o)}{\sqrt{K_p(x,x)+K_p(o,o)-2K_p(x,o)}}\right)\Bigg)\\
\end{aligned}
\end{equation}
and
\begin{equation}
\label{eq:millsthird1142sec1}
\begin{aligned}
&E[\nu(x)I_{\{|\nu(o)-\nu(x)|\leq \sigma\}}]=\\
&=\mu_p(x)-\mu_p(x)\Bigg(2-\Phi\left(\tfrac{(\mu_p(x)+\sigma-\mu_p(o))}{\sqrt{K_p(x,x)+K_p(o,o)-2K_p(x,o)}}\right)\\
&-\Phi\left(\tfrac{(-\mu_p(x)+\sigma+\mu_p(o))}{\sqrt{K_p(x,x)+K_p(o,o)-2K_p(x,o)}}\right)\Bigg)\\
& + \tfrac{K_p(x,o)-K_p(x,o)}{\sqrt {K_p(x,x)+K_p(o,o)-2K_p(x,o)}}\\
&\Bigg(\phi\left(\tfrac{-\mu_p(o)+\sigma+\mu_p(x)}{\sqrt{K_p(x,x)+K_p(o,o)-2K_p(x,o)}}\right)\\
&-\phi\left(\tfrac{\mu_p(o)+\sigma-\mu_p(x)}{\sqrt{K_p(x,x)+K_p(o,o)-2K_p(x,o)}}\right)\Bigg)\\
\end{aligned}
\end{equation}

\paragraph{Proof of Proposition \ref{prop:3}}
If $\bR$ has \textbf{no uncertainty} and $\bH$ is rational, then the expected payoffs can be computed from \eqref{eq:DEFexp1}, \eqref{eq:IMMexp1} and \eqref{eq:OFFexp1}.
  The values are
$E[DEF]=\max(\mu_p(x),\mu_p(o))$, 
  $E[IMM]=\mu_p(x)$ and $E[OFF]=\mu_p(o)$. Therefore, DEF is always optimal.
  
   If $S$ is \textbf{bounded-rational} and $R$ has \textbf{no uncertainty}, then $\sigma>0$. We consider three cases:
   (1) $\mu_p(x)>\mu_p(o)+\sigma$;
   (2) $\mu_p(o)>\mu_p(x)+\sigma$;
   (3) otherwise.

In case (1), the payoffs are:
   $E[DEF]=\mu_p(x)$,
  $E[IMM]=\mu_p(x)$ and $E[OFF]=\mu_p(o)$,
   Therefore,    DEF is  optimal. A similar results holds in case (2). In case (3), $E[DEF]=\{\mu_p(x)-\epsilon,\mu_p(o)-\epsilon\}$. Under the condition (A) or (B),  DEF will alwys be dominated.
If $S$ is \textbf{rational} and $R$ has \textbf{uncertainty}, then 
  DEF is optimal as in Proposition \ref{prop:1}
The last case follows directly from Lemma \ref{lem:2}.

\paragraph{Proof of Proposition \ref{prop:honestr}}
The only case where the content of the message is important is when DEF is not optimal. In this case, $\bR$ makes a decision autonomously.

Whenever DEF is not optimal, the best action can be either IMM$(x)$ if $\mu_p(x) > \mu_p(o)$ or OFF if $\mu_p(o) > \mu_p(x)$. Therefore, if $\bH$ sends a biased message such that $\bR$ estimates $\mu_p(x) > \mu_p(o)$ when, in reality, $\nu(x) < \nu(o)$, then $\bR$ would choose an action that is not optimal for $\bH$.

\paragraph{Proof of Proposition \ref{prop:4}}
This follows from Proposition \ref{prop:1} for the cases when $\boldsymbol{\nu}(o)$ and $\boldsymbol{\nu}(x)$ are comparable (i.e., one dominates the other).
If $S$ is \textbf{rational} and $R$ has \textbf{no uncertainty}, then if $\boldsymbol{\nu}(o)$ and $\boldsymbol{\nu}(x)$ are comparable, the payoff for DEF$(x)$ will be the best between $\boldsymbol{\nu}(o)$ and $\boldsymbol{\nu}(x)$ and, therefore DEF is not dominated.

If $S$ is \textbf{bounded-rational} and $R$ has \textbf{no uncertainty}, then if $\boldsymbol{\nu}(o)$ and $\boldsymbol{\nu}(x)$ are comparable, the payoff for DEF$(x)$ will never be optimal, because it will be $p\boldsymbol{\nu}(o)+(1-p)\boldsymbol{\nu}(x)$.
If $S$ is \textbf{rational} and $R$ has \textbf{ uncertainty}, then if $\boldsymbol{\nu}(o)$ and $\boldsymbol{\nu}(x)$ are comparable, the payoff for DEF$(x)$ will always be optimal as shown  in Proposition \ref{prop:1}. If $S$ is \textbf{bounded-rational} and $R$ has \textbf{ uncertainty}, the best decision depends on the specific case.

\bibliographystyle{ieeetr}
\bibliography{biblio}

\end{document}